\documentclass[12pt]{article}
\usepackage[margin=1in]{geometry}
\usepackage[margin=.6in]{caption}

\usepackage{amsmath,amssymb,amsthm}
\usepackage{graphicx}
\usepackage{color}
\usepackage{bbm}
\usepackage{dsfont}
\usepackage{enumerate,enumitem}
\usepackage{verbatim}
\usepackage{tikz}
\usepackage{url}

\usepackage{hyperref}

\usepackage{authblk}

\newtheorem{theorem}{Theorem}
\numberwithin{theorem}{section}
\newtheorem{proposition}[theorem]{Proposition}
\newtheorem{lemma}[theorem]{Lemma}
\newtheorem{corollary}[theorem]{Corollary}
\newtheorem{computation}[theorem]{Computation}

\theoremstyle{definition}
\newtheorem{definition}[theorem]{Definition}
\newtheorem{example}[theorem]{Example}
\newtheorem{problem}[theorem]{Problem}

\theoremstyle{remark}
\newtheorem{observation}[theorem]{Observation}
\newtheorem{remark}[theorem]{Remark}

\newcommand{\R}{\mathbb{R}}
\newcommand{\PP}{\mathbb{P}}
\newcommand{\RBM}{\operatorname{RBM}}
\newcommand{\vol}{\operatorname{vol}}
\newcommand{\conv}{\operatorname{conv}}
\newcommand{\sgn}{\operatorname{sgn}}

\newcommand{\X}{\mathcal{X}}
\newcommand{\Pcal}{\mathcal{P}}

\newcommand{\Ccal}{\mathcal{C}}
\newcommand{\Acal}{\mathcal{A}}

\newcommand{\Mcal}{\mathcal{M}}
\newcommand{\Zcal}{\mathcal{Z}}
\newcommand{\Gcal}{\mathcal{G}}

\newcommand{\Hcal}{\mathcal{H}}
\newcommand{\Ycal}{\mathcal{Y}}
\newcommand{\Y}{\mathcal{Y}}
\newcommand{\Ca}{\mathfrak{C}_{\text{\normalfont aff}}}
\newcommand{\Cn}{\mathfrak{C}}
\newcommand{\Xcal}{\mathcal{X}}

\newcommand{\up}{\operatorname{up}}
\newcommand{\down}{\operatorname{down}}

\begin{document}

\title{When Does a Mixture of Products \\Contain a Product of Mixtures?} 

\author[1,2]{Guido F. Mont\'ufar\thanks{\href{mailto:montufar@mis.mpg.de}{montufar@mis.mpg.de}}}
\author[2]{Jason Morton\thanks{\href{mailto:morton@math.psu.edu}{morton@math.psu.edu}}}

\affil[1]{Max Planck Institute for Mathematics in the Sciences\\ Inselstrasse 22, 04103 Leipzig, Germany.}
\affil[2]{Department of Mathematics, Pennsylvania State University\\ University Park, PA 16802, USA.}

\maketitle
\thispagestyle{empty}

\begin{abstract}
We derive relations between theoretical properties of restricted Boltzmann machines (RBMs), popular machine learning models which form the building blocks of deep learning models, and several natural notions from discrete mathematics and convex geometry.  We give implications and equivalences relating RBM-representable probability distributions, perfectly  reconstructible inputs, Hamming modes, zonotopes and zonosets, point configurations in hyperplane arrangements, linear threshold codes, and multi-covering numbers of hypercubes. As a motivating application, we prove results on the relative representational power of mixtures of product distributions and products of mixtures of pairs of product distributions (RBMs) that formally justify widely held intuitions about distributed representations. In particular, we show that a mixture of products requiring an exponentially larger number of parameters is needed to represent the probability distributions which can be obtained as products of mixtures. 

\end{abstract}


\section{Introduction}\label{secintro}

Two basic ways of combining probability distributions are mixtures, i.e., convex combinations, and Hadamard products, i.e., renormalized entry-wise products. 
Fixing the number of parameters, we may ask: are Hadamard products of small mixtures better than mixtures at approximating interesting or complex probability distributions? 
The general intuition among practitioners is that using Hadamard products allows for more modeling power. 
We compare two canonical representatives of these model classes: mixtures of product distributions, called na\"ive Bayes models, and Hadamard products of mixtures of pairs of product distributions, called restricted Boltzmann machines (RBMs). 
The mixture of products model $\Mcal_{n,k}$ is the union of the convex hulls of all choices of $k$ joint distributions for $n$ independent binary variables (Definition~\ref{def:mixtureofproducts}). 
The restricted Boltzmann machine model $\RBM_{n,m}$ is the union of the Hadamard products of all choices of $m$ mixtures of pairs of joint distributions for $n$ independent binary variables (Definition~\ref{def:RBM}). 
Both are graphical probability models with hidden variables and bipartite graphs, see Figure~\ref{RBMandMixt}. 
Besides defining probability distributions on their visible states, these graphical models define conditional distributions between visible and hidden states, which makes them interesting in the context of learning representations.  
\begin{figure}
\begin{center}
\setlength{\unitlength}{1cm}
\begin{picture}(10.5,3.7)(.75,.0)
\put(6,0){\includegraphics[clip=true,trim=5.3cm 22.04cm 11.5cm 1cm,width=5.5cm]{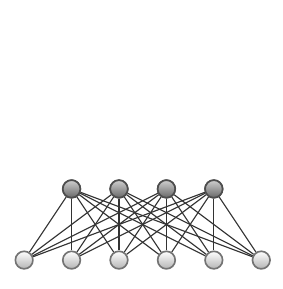}}
\put(0,0){\includegraphics[clip=true,trim=5.3cm 19.6cm 11.5cm 6cm,width=5.5cm]{RBMandMixt2BW.pdf}} 
\put(0,2.4){\begin{minipage}{5.5cm}\centering $k$\end{minipage}}
\put(0,3){\begin{minipage}{5.5cm}\centering Mixture of  Products \end{minipage}}
\put(6,3){\begin{minipage}{5.5cm}\centering Product of  Mixtures of Products (RBM)\end{minipage}}
\put(0,-.2){\begin{minipage}{5.5cm}\centering $\Mcal_{6,k}$\end{minipage}}
\put(6,-.2){\begin{minipage}{5.5cm}\centering $\RBM_{6,4}$\end{minipage}}
\end{picture}
\end{center}
\caption{Graphical representation of a mixture of products and a product of mixtures. 
The dark nodes represent hidden units and the light nodes represent visible units. 
}\label{RBMandMixt}
\end{figure}
This paper is the result of analyzing following problem. 
\begin{problem}\label{problem:When}
When does the mixture of product distributions $\Mcal_{n,k}$ contain the product of mixtures of product distributions $\RBM_{n,m}$, and vice versa? 
\end{problem}

In answer to the title question, our results in Sections~\ref{section:perspectives} and \ref{section:when} imply the following.
\begin{theorem}\label{thmsummaryans}
The number of parameters of the smallest mixture of products $\Mcal_{n,k}$ containing the product of mixtures $\RBM_{n,m}$ grows exponentially in the number of parameters of the latter for any fixed ratio $0 < m/n < \infty$. 
More precisely, the smallest $k$ such that  $\Mcal_{n,k}$ contains $\RBM_{n,m}$ is bounded by $\frac{3}{4}n \leq \log_2(k) \leq n-1$ when $m\geq n$, by $\frac{3}{4}n \leq \log_2(k) \leq m$ when $\frac{3}{4}n\leq m\leq n$, and satisfies $\log_2(k)=m$ when $m\leq \frac{3}{4}n$. 
\end{theorem}
The solution is of order $\log_2(k)= \Theta(\min\{m,n\})$. 
Hence the smallest mixture of products model that contains a given RBM model is as large as one can possibly expect, having near to one mixture component per joint state of the RBM hidden units, or otherwise containing every possible probability distribution. 
See Figure~\ref{solnm} for an illustration of the result. 
Theorem~\ref{thmsummaryans} is based on the more technical Theorem~\ref{notinclpropodos}. 
In Theorem~\ref{mixtnotinrbm} we show a complementary result stating that, although RBMs naturally contain small mixture models, in general they do not contain mixture models that match their dimension. 

To approach Problem~\ref{problem:When}, we study the sets of modes (Hamming-local maxima) of probability distributions that can be represented as mixtures of product distributions and as RBMs. 
We consider the following problems, showing in many cases that they are equivalent or equivalent after adding some necessary conditions.

\begin{problem}\label{problem1}
What sets of length-$n$ binary vectors are 
\begin{enumerate}
\item the modes or strong modes (Hamming-local maxima) of probability distributions represented by an RBM with $m$ hidden units?
\item perfectly reconstructible (given a vector in the set, choosing the most likely hidden state, then the most likely visible state, returns the given vector) by an RBM with $m$ hidden units? 
\item the outputs of $n$ linear threshold functions with $m$ inputs? 
\end{enumerate}
\end{problem}

We find that probability distributions with many strong modes (for example, probability distributions strictly supported on the binary vectors with even or odd number of ones), can be represented far more compactly by RBMs than by mixtures of products. 
Modes are described by linear inequalities of the form $p(x)>p(x')$ and can be used to derive polyhedral approximations of probability models. 
As it turns out, the analysis of modes is closely related to binary classification problems (separation of vertex sets of hypercubes by hyperplane arrangements), and leads to problems such as the following.

\begin{problem}
What is the smallest arrangement of hyperplanes, if one exists, that slices each edge of a hypercube a given number of times? 
\end{problem}

\begin{figure}
\begin{center}
\setlength{\unitlength}{1.1cm}
\begin{picture}(11.5,4.55)(0,.7)
\put(0,0){\includegraphics[clip=true, trim=2cm 20.2cm 5cm 1.8cm, scale=1.1 ]{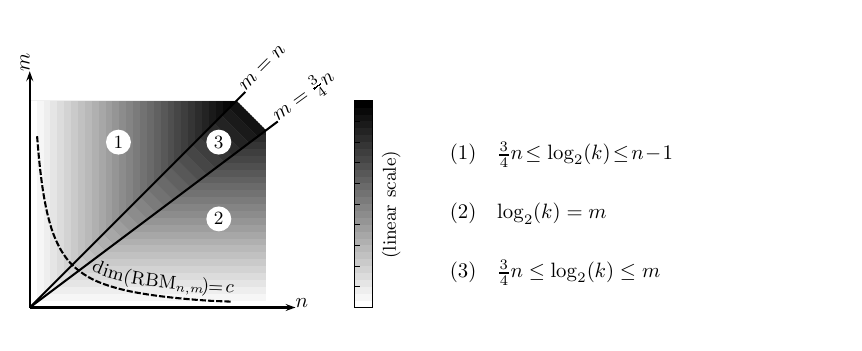}} 
\end{picture}
\end{center}
\caption{Smallest mixtures of products that can represent an RBM. 
Shown is the heat map of the logarithm of $k(n,m) =\min\{ k'\in\mathbb{N} \colon \Mcal_{n,k'}\supseteq \RBM_{n,m}\}$, depending on $n,m\in\mathbb{N}$. The domain of this  function has three regions, each with approximately linear behavior (Theorems~\ref{thmsummaryans} and~\ref{notinclpropodos}). 
The model $\RBM_{n,m}$ has $nm+n+m$ parameters. 
Fixing $nm+n+m=c$ (the dashed hyperbola), the RBMs which are hardest to represent as mixtures of product distributions are those with $m/n \approx 1$.
}\label{solnm}
\end{figure}

We consider the following six properties of sets of binary vectors, and derive relations between them, summarized below in Theorem~\ref{thm:equivalences}. 

\begin{definition} 
\label{def:summaryofproperties}
\mbox{}
Let $n$ and $m$ be two non-negative integers and let $\Ccal$ be a subset of $\{0,1\}^n$.
\begin{itemize}
\item {\bf LTC($n,m,\Ccal$)}: The set $\Ccal$ is an $(n,m)$-linear threshold code, i.e., the image of $n$ linear threshold functions with $m$ inputs (Definition \ref{defLTC}).
\item {\bf HP($n,m,\Ccal$)}: There exists an arrangement $\Acal$ of $n$ hyperplanes in $\R^m$ such that the vertices of the $m$-dimensional unit cube intersect exactly the $\Ccal$-cells of $\Acal$ (Definition~\ref{def:hyperplanearrangement}).
\item {\bf ZP($n,m,\Ccal$)}: There is an $m$-zonoset (i.e., the affine image of the vertices of an $m$-cube) in $\R^n$ which intersects exactly the $\Ccal$-orthants of $\R^n$ (Definition~\ref{def:zonoset}). 
\item {\bf SM($n,m,\Ccal$)}: An RBM  with $n$ visible and $m$ hidden nodes can represent a distribution with set of strong modes $\Ccal$ (Definition~\ref{strongmodedef}). 
\item {\bf PR($n,m,\Ccal$)}: The set $\Ccal$ is the set of perfectly reconstructible vectors of an RBM with $n$ visible and $m$ hidden units (Definition \ref{def:perfectlyreconstructible}).
\item {\bf SP($n,m,\Ccal$)}: An RBM with $n$ visible and $m$ hidden units can represent a distribution which is strictly positive on $\Ccal$ and zero elsewhere.
\end{itemize}
\end{definition}

We derive implications among the properties {\bf LTC}, {\bf PR}, {\bf HP}, {\bf ZP}, {\bf SM}, and {\bf SP} in two cases: 
the set $\Ccal$ is arbitrary, and $\Ccal$ consists of vectors which are at least Hamming distance~$2$ apart. 

\begin{figure}
\begin{center}
\begin{tikzpicture}[scale=1.5]
\node (SP) at (1,3) {\bf SP};
\node (SM) at (0,2) {\bf SM};
\node (PR) at (2,2) {\bf PR};
\node (LTC) at (1,1) {\bf LTC};
\node (ZP) at (2,1) {\bf ZP};
\node (HP) at (0,1) {\bf HP};
\draw[->,dashed, very thick] (SP) -- (SM);
\draw[->,dashed, very thick] (SP) -- (PR);
\draw[->, very thick] (SM) -- (LTC);
\draw[->, very thick] (PR) -- (LTC);
\draw[->,dotted, very thick] (LTC) -- (SP);
\draw[<->, very thick] (ZP) -- (LTC);
\draw[<->, very thick] (HP) -- (LTC);

\node at (1.3,2) {$\ell_1$};
\node at (2.5,2.5) {$d_H(\Ccal)\geq 2$};
\node at (-.5,2.5) {$d_H(\Ccal)\geq 2$};
\end{tikzpicture}
\end{center}
\caption{Illustration of the implications in Theorem \ref{thm:equivalences}} \label{fig:equivalences}
\end{figure}
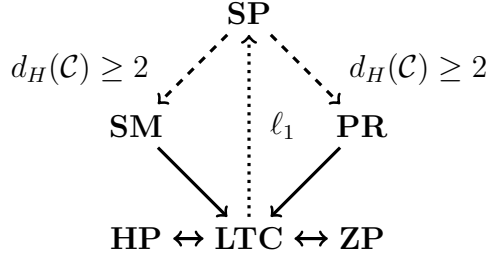

\begin{theorem} 
\label{thm:equivalences} 
Let $n$ and $m$ be two non-negative integers and let $\Ccal$ be a subset of $\{0,1\}^n$. 
\begin{enumerate}
\item The properties {\bf LTC}, {\bf HP,} and {\bf ZP} are equivalent. 
\item If $\Ccal$ satisfies {\bf PR} or {\bf SM}, then it is contained in an {\bf LTC} set. 
\item If the vectors in $\Ccal$ are at least Hamming distance $2$ apart, then {\bf SP} implies both {\bf SM} and {\bf PR}. 
\item If the vectors in $\Ccal$ are at least Hamming distance $2$ apart and $\Ccal$ satisfies an $\ell_1$ property (see Theorem~\ref{zonotopecondpropo}), then {\bf LTC} implies {\bf SP}.  
\end{enumerate}
\end{theorem}
Figure~\ref{fig:equivalences} illustrates the result. 
The proof is given in Section \ref{section:summary} by combining results from Section~\ref{section:perspectives}.

\bigskip

Section~\ref{section:mixturesof} contains basic definitions and background on mixtures of product distributions and RBMs. 
Section~\ref{section:perspectives} discusses geometric perspectives on statistical models and inference, elaborated in various subsections. 
Section~\ref{section:inference} discusses inference functions, distributed representations, and reconstructability. 
Section~\ref{section:modes} discusses the concept of modes and polyhedral approximations of probability models.  
Section~\ref{section1} covers the sets of modes of probability distributions realizable as mixtures of product distributions (Theorem~\ref{sobrecomplmix}). 
Section~\ref{section2} makes a few initial observations on the sets of modes of probability distributions realizable by RBMs. 
In Section~\ref{section:zonosets} these sets are related to zonosets and hyperplane arrangements (Theorem~\ref{zonotopecondpropo}), and in Section~\ref{section:LTC} to linear threshold codes. 
Section~\ref{section6} discusses multi-covering numbers; the smallest hyperplane arrangements slicing each edge of a hypercube a given number of times. 
Section~\ref{section:summary} contains the proof of Theorem~\ref{thm:equivalences}. 
Turning to the motivating questions, Section~\ref{section:when} contains our analysis of Problem~\ref{problem:When}, 
treating the inclusion of RBMs in mixture models in Section~\ref{section3}, and the reverse inclusion in Section~\ref{section5}. 
Section~\ref{sec:discussion} offers a discussion. 

\section{Mixtures of products and products of mixtures} \label{section:mixturesof}

Let $\mathcal{P}_n$ denote the $(2^n - 1)$-dimensional simplex of joint probability distributions of $n$ binary variables. 
This is the set of vectors $p\in \R^{2^n}$ with entries $p(x)\geq 0$, $x=(x_1,\ldots,x_n)\in\{0,1\}^n$, 
satisfying $\sum_{x\in\{0,1\}^n}p(x)=1$. 
For any given $p\in\Pcal_n$, the marginal distribution of the $i$-th variable is the vector $p_i\in\Pcal_1$ with entries $p_i(x_i) = \sum_{(x_1, \dots, x_{i-1},x_{i+1} \dots x_n) \in \{0,1\}^{n-1}}p(x)$, $x_i\in\{0,1\}$. 

Let $\Mcal_{n,1}$ denote the $n$-dimensional set of joint probability distribution of $n$ independent binary variables. 
This is the set of distributions $p\in\Pcal_n$ that factorize as $p(x)=p_1(x_1)\cdots p_n(x_n)$, $x=(x_1,\ldots, x_n)\in\{0,1\}^n$, called product distributions. 
One can regard these distributions as the $n$-way $2\times\cdots\times 2$ tables of the form $p_1\otimes\cdots\otimes p_n$, $p_i\in\Pcal_1$, $i\in[n]$,  where $\otimes$ denotes the tensor product. 
We note that $\Mcal_{n,1}$ is the closure of the exponential family 
$p_B(x)= \frac{1}{Z(B)}\exp(B^\top x)$, $x\in\{0,1\}^n$, with natural parameter $B\in\R^n$ and normalization function $Z(B)=\sum_{y\in\{0,1\}^n}\exp( B^\top y)$. Closure is needed in order to include probability distributions with vanishing entries. 

\begin{definition}\label{def:mixtureofproducts}
The {\em $k$-mixture of product distributions of $n$ binary variables}, denoted $\Mcal_{n,k}$, is the set of distributions on $\{0,1\}^n$ expressible as convex combinations $p = \sum_{i\in[k]}\lambda_i q^{(i)}$, where $\lambda_i \geq 0$, $\sum_{i\in[k]}\lambda_i = 1$, and $q^{(i)}\in\Mcal_{n,1}$ for all $i\in[k]$. 
\end{definition}
Up to positive scalar multiples, $\Mcal_{n,k}$ corresponds to the set of $n$-way $2\times\cdots\times 2$ tables with non-negative rank at most $k$. 
The Zariski closure of $\Mcal_{n,k}$ in complex projective space is the $k$-th secant variety of the $n$-th Segre product of $\PP^1$s and has the same dimension as $\Mcal_{n,k}$. 
As it turns out, this is the dimension expected from counting parameters, equal to $\min\{nk + (k-1),2^n - 1\}$, except when $(n,k) = (4,3)$, in which case it has dimension $13$ instead of $14$. See~\cite{Catalisano2011}, which answered this century-old question in algebraic geometry. 

The set $\Mcal_{n,k}$ is equal to the probability simplex $\Pcal_n$ if and only if $k\geq 2^{n-1}$, see~\cite{Montufar2010a}. 
In particular, the smallest $\Mcal_{n,k}$ that equals $\Pcal_n$ has $2^{n-1}(n+1)-1$ parameters. 

\begin{definition}\label{def:RBM}
The {\em RBM model} with $n$ visible and $m$ hidden binary units, denoted $\RBM_{n,m}$, is the closure of the set of distributions on $\{0,1\}^n$ of the form 
\begin{equation}
p(x) = \frac{1}{Z(W,B,C)} \sum_{h\in\{0,1\}^m}\exp(h^\top W x + B^\top x + C^\top h)\quad\text{ for all $ x\in\{0,1\}^n$},
\label{eq:rbmdef}
\end{equation}
where $W \in \R^{m\times n}$ is a matrix of interaction weights between hidden and visible units (with state vectors $h$ and $x$, respectively), $B \in \R^n$ is a vector of biases of the visible units, $C \in \R^m$ is a vector of biases of the hidden units, and $Z(W,B,C)=\sum_{x\in\{0,1\}^n}\sum_{h\in\{0,1\}^m}\exp(h^\top W x + B^\top x + C^\top h)$ is a normalization function. 
\end{definition} 
An RBM is a {\em  product of experts}~\cite{Hinton99productsof}; each hidden unit corresponds to an expert which is a mixture of two product distributions~\cite{Cueto2010}. For completeness we provide a proof of this statement:  
\begin{proposition}
Each RBM distribution of the form~\eqref{eq:rbmdef} is a renormalized entry-wise product of positive mixtures of pairs of positive product distributions, and vice versa. 
\end{proposition}
\begin{proof}	 
Each distribution $p$ of the form~\eqref{eq:rbmdef} can be written as a renormalized entry-wise product 
$$p(x)= \frac{q^{(1)}(x)\cdots q^{(m)}(x)}{\sum_{y\in\{0,1\}^n} q^{(1)}(y)\cdots q^{(m)}(y)}\quad\text{for all } x\in\{0,1\}^n,$$ 
where $q^{(j)} = \lambda_j p_{A_j} + (1-\lambda_j) p_{A'_j} \in\Mcal_{n,2}$ is a positive mixture of positive product distributions, with $\lambda_j\in(0,1)$, $p_{A_j}=\exp(A_j^\top x)/Z(A_j)$ and $p_{A'_j}=\exp({A'_j}^\top x)/Z(A'_j)$, for all $j\in[m]$. 
To see this, note that 
\begin{align*}
\sum_{h\in\{0,1\}^m}\exp(h^\top W x + B^\top x + C^\top h) 
=& \prod_{j\in[m]}\sum_{h_j\in\{0,1\}}\exp(h_j W_{j} x + \tfrac{1}{m}B^\top x + C_j h_j) \\
=& \prod_{j\in[m]} \left(\exp(\tfrac{1}{m}B^\top x) + \exp( C_j) \exp(  ( W_{j} +\tfrac{1}{m} B^\top) x) \right)\\
\propto& \prod_{j\in[m]} \left( \lambda_j p_{A_j}(x) + (1-\lambda_j) p_{A'_j}(x) \right), 
\end{align*}
where $W_j$ is the $j$-th row of $W$, $A_j=B/m$, $A'_j=W_j^\top + (B/m)$, and $\lambda_j =  Z(A_j) /( Z(A_j)+ Z(A'_j) \exp(C_j) )$. 

Conversely, each renormalized entry-wise product of positive mixtures of pairs of positive product distributions is of the form~\eqref{eq:rbmdef}. To see this, note that, for any choice of $A_j, A'_j\in \R^n$ and $\lambda_j\in(0,1)$, 
\begin{align*}
\prod_{j\in[m]} \Big( \lambda_j p_{ A_j } (x) + (1-\lambda_j)& p_{ A_j' }(x) \Big) \\
=& \prod_{j\in[m]} \left(  \frac{\lambda_j}{Z(A_j)}\exp(A_j^\top x) +  \frac{(1-\lambda_j)}{Z(A'_j)}\exp( {A'_j}^\top  x) \right) \\
\propto& \prod_{j\in[m]}\exp({A_j}^\top x) \left( 1  +  \frac{Z(A_j)}{\lambda_j}\frac{ (1-\lambda_j) }{  Z(A'_j)}\exp( ( A'_j - A_j)^\top  x) \right) \\
\propto& \sum_{h\in\{0,1\}^m}\exp(h^\top W x + B^\top x + C^\top h) , 
\end{align*}
where $B = \sum_{j\in[m]}A_j$, $W_j = A'_j  - A_j$, and $C_j =\log \left({ Z(A_j) (1-\lambda_j)} / { \lambda_j Z(A'_j)}\right)$.  
\end{proof}

An RBM can also be seen as a set of restricted mixtures of product distributions; 
each $p\in\RBM_{n,m}$ is a mixture of $2^m$  product distributions, namely the conditionals 
$$p(x|h) = \frac{\exp((h^\top W +B^\top)x) }{ \sum_{y\in\{0,1\}^n}\exp((h^\top W+B^\top) y)} = p_{W^\top h+B}(x) \quad \text{for all $h\in\{0,1\}^m$}.$$  

In general, the dimension of the mixture model $\Mcal_{n,2^m}$ is much larger than that of $\RBM_{n,m}$. 
The set $\RBM_{n,m}$ is known to have dimension $nm+n+m$ when $m < 2^{n-\lceil\log_2(n+1)\rceil}$, 
and $2^n - 1$  when $m\geq 2^{n-\lfloor\log_2(n+1)\rfloor}$, see~\cite{Cueto2010}. 
In addition, it is known that $\RBM_{n,m}$ equals $\Pcal_n$ whenever $m \geq 2^{n-1}-1$, see~\cite{Montufar2011}. 
It is not known if the latter bound is always tight, but it shows that the smallest $\RBM_{n,m}$ that equals $\Pcal_n$ has not more than $2^{n-1}(n+1)-1$ parameters, and hence not more than the smallest mixture of products model. 

We will show that the sets of probability distributions representable by RBMs and mixtures of products are quite different. 
The intersection of both model classes has been studied in~\cite{MonRauh2011}, 
where it is shown that $\RBM_{n,m}$ contains any mixture of $m+1$ product distributions with disjoint supports, 
and hence that the intersection $\RBM_{n,m} \cap \Mcal_{n,m+1}$ has dimension of order at least $m n + m + 2 n + 1 - (m-1)\log_2(m+1)$. 
Typically RBMs have many binary hidden variables and mixtures of products models have a single multivalued hidden variable. 
The transition between these two limit cases has been studied in~\cite{montufar2013discrete}, 
focusing on the Kullback-Leibler model approximation errors and the model dimension. 

\section{Geometric perspectives}\label{section:perspectives}

In this section we present five points of view on the families of probability distributions defined in the previous section. 
We consider inference functions and hidden representations defined by these models (Section \ref{section:inference}), 
modes and strong modes of their marginal distributions (Sections \ref{section:modes}, \ref{section1}, and \ref{section2}), 
zonosets, hyperplane arrangements, and linear threshold codes that capture their combinatorial structure (Sections~\ref{section:zonosets} and~\ref{section:LTC}), 
and the resulting multi-covering numbers of hypercubes (Section \ref{section6}). 

Each point of view comes with particular set of related tools and implications for the capabilities of RBMs and competing models. We determine how these five concepts are related, which imply which, and how properties such as the number of strong modes in a marginal distribution translate into each perspective. 
These observations are summarized in Section~\ref{section:summary}, and in Section~\ref{section:when}, where they are applied to distinguish mixtures of products and products of mixtures.

\subsection[Inference functions, representations, reconstructability]{Inference functions, distributed representations, \newline reconstructability}\label{section:inference}

\begin{figure}
\begin{center}
\setlength{\unitlength}{1cm}
\begin{picture}(10.7,4.6)(0.4,-0.2)
\put(0,0){
\includegraphics[clip=true,trim=2cm 18.2cm 6cm 1.4cm,width=11cm]{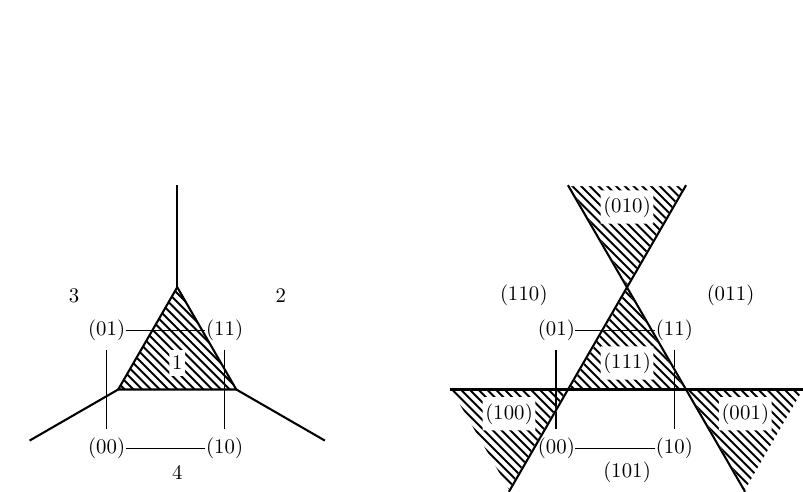}
}
\end{picture}
\end{center}
\caption{Inference regions of $\Mcal_{2,4}$ (left) and $\RBM_{2,3}$ (right), for a choice of parameters.  
}\label{inferenceregions}
\end{figure}

Hinton~\cite{Hinton99productsof} discusses advantages of products of experts (Hadamard products of probability models) over  mixtures of experts (mixtures of probability models), for modeling ``high-dimensional data which simultaneously satisfies many low-dimensional constraints.'' In products of experts models, each expert can individually  ensure that one constraint is satisfied. 
In the case of RBMs, each hidden unit linearly divides the input space according to its preferred state given the input, which results in a {\em multi-clustering}, or a partition of the input space into cells where different joint hidden states are most likely. 
Inference of the most likely hidden state given an input produces a distributed encoding or {\em distributed representation} of the input vector, as discussed by Bengio in~\cite[Section~5.3]{Bengio-2009}. 

\begin{definition}
The {\em inference function} of a probability model $p_\theta(v,h)$ with parameter $\theta\in\R^N$ `explains' each value of $v$ by the most likely value of $h$ according to $\up_\theta \colon v\mapsto \operatorname{argmax}_h p_\theta(h|v)$. This defines a partition of the input space into the preimages of all possible outputs, called {\em inference regions}. 
\end{definition}

Inference functions provide a combinatorial view on the corresponding probability models. 
They appear naturally in the context of {\em tropicalization}, 
where they correspond to the linear regions of certain piecewise linear approximations of algebraic varieties and serve to estimate their dimension. This approach has been studied in~\cite{Draisma} in the context of secants and in~\cite{Cueto2010, montufar2013discrete, Kronecker} in the context of RBMs. 

For each choice of parameters $W\in\R^{m\times n}$, $B\in\R^n$, $C\in\R^m$, the model $\RBM_{n,m}$ defines the inference function 
\begin{equation*}
\up_{W,B,C}\colon\;  \R^n \supset\{0,1\}^n\to \{0,1\}^m;\; v\mapsto \operatorname{argmax}_{h\in\{0,1\}^m}  h^\top (W v + C) \;.
\end{equation*} 
The visible state $v$ is explained by the hidden state $h$ which satisfies $\sgn(W v + C)=\sgn(h^\top-\tfrac{1}{2}\mathds{1})$, where $\mathds{1}:= (1,\ldots,1)$. 
There may be several explanations for a given observation, but generically there is only one. 
Geometrically, the input space $\R^n$ is partitioned into the preimages of the orthants of $\R^m$ by the affine map $\psi\colon \R^n\to\R^m; v\mapsto Wv+C$. 
This partition corresponds to the intersection of an affine space and the normal fan of an $m$-cube (the orthants of $\R^m$). 
The number of inference regions can be as large as $\Ca(m,d)=\sum_{i=0}^{d}{m \choose i}$, which is the number of orthants of $\R^m$ intersected by a generic $d$-dimensional affine subspace, where $d\leq \min\{n,m\}$ is the rank of $W$. When the rank of $W$ is less than $m$ (for example, when $m>n$), then the image of the map $\psi$ does not intersect all orthants of $\R^m$ and there are `empty' inference regions, i.e., states $h$ which are not the explanation of any input vector~$v$. 

The mixture model $\Mcal_{n,k}$, on the other hand, defines, for any choice of the mixture weights $\lambda_i$ and the natural parameters of each mixture component $B_{i}\in\R^n$ for $i\in[k]$, an inference function 
\begin{equation*}
\up_{\lambda,B}\colon\; \R^n \supset\{0,1\}^n\to \{1,\ldots, k\};\; v\mapsto \operatorname{argmax}_{i\in[k]} (B_{i}^\top v -\log(Z(B_i))+ \log (\lambda_{i})),  
\end{equation*} 
where $Z(B_{i})=\sum_{v\in\{0,1\}^n} \exp(B_{i}^\top v)$. 
In this case, the input space $\R^n$ is partitioned into the at most $k$ regions of linearity of the function $v\mapsto \max\{ B_{i}^\top v - \log (Z(B_{i})) +\log(\lambda_i)\colon i\in[k]\}$. 
This partition corresponds to the intersection of an affine space and the normal fan of a $(k-1)$-simplex. 

Figure~\ref{inferenceregions} shows an example of inference regions in $\{0,1\}^2\subset\R^2$ defined by $\Mcal_{2,4}$ (left panel) and $\RBM_{2,3}$ (right panel), for some specific parameter values. Both models have $7$ parameters and are universal approximators of distributions on $\{0,1\}^2$, but they define very different inference regions. 

For a fixed input space of dimension $n$, the number of inference regions in $\R^n$ that can be realized by $\RBM_{n,m}$ is of order $\Theta({m\choose \min\{n,m\}})$, which is exponential in the number of parameters of the model, whereas the number of inference regions that can be realized by $\Mcal_{n,k}$ is linear in the number of parameters of the model.  A function $g\colon \R_+\to\R_+$ is of order $\Theta(f)$ if there exist positive constants $C,C'$ and $n_0$ such that $Cf(n) \leq g(n) \leq C'f(n)$ for all $n\geq n_0$~\cite{knuth1976big}. 
Distributed representations can, in principle, learn different explanations to a number of observations that is exponential in the number of model parameters, see~\cite{Bengio-2009}. 

Now we discuss reconstructability. 
Similarly to the $\up_\theta$ inference function, a model $p_\theta(v,h)$ defines a $\down_\theta$ inference function, which outputs the most likely visible state $\operatorname{argmax}_v p_\theta(v|h)$ given a hidden state $h$.  

\begin{definition}\label{def:perfectlyreconstructible}
Given a probability model $p_\theta(v,h)$ on $v\in\Xcal$ and $h\in\Ycal$, 
a collection of states $\Ccal\subseteq\Xcal$ is {\em perfectly reconstructible} if there is a choice of the parameter $\theta$ for which $\down_\theta(\up_\theta(v))=v$ for all $v\in\Ccal$. 
\end{definition}
The ability to reconstruct input vectors is sometimes used to evaluate the performance of RBMs in practice, since it can be tested more cheaply than the probability distributions they represent. 
The reconstructability of input vectors can also be used to define training algorithms, like in the case of auto-encoders. 

When writing the joint probabilities $(p_\theta(v,h))_{v,h}$ as a matrix with rows labeled by $h\in\Ycal$ and columns by $v\in\Xcal$, 
a set $\Ccal\subseteq\Xcal$ is perfectly reconstructible iff there is a choice of the model parameter $\theta$ for which $p_\theta(v, \up_\theta(v))$ is the unique maximal entry in the $\up_\theta(v)$-row (and in the $v$-column) for all $v\in\Ccal$. 
For the model $\RBM_{n,m}$, this is the case exactly when for each $v\in\Ccal\subseteq\{0,1\}^n$ there is an $h_v\in\{0,1\}^m$ with $\sgn(W v + C)=\sgn(h_v-\frac{1}{2}\mathds{1})$ and $\sgn(h_v^\top W +B^\top)=\sgn(v-\frac{1}{2}\mathds{1})$. 
\begin{example}
If $n,m\geq k$, all cylinder subsets of $\{0,1\}^n$ of dimension $k$ are perfectly reconstructible by $\RBM_{n,m}$. 
A $k$-dimensional cylinder subset of $\{0,1\}^n$ is a set of the form $\{ x\in\{0,1\}^n\colon x_i=y_i \text{ for all } i\in I \}$, where $I$ is a subset of $[n]$ of cardinality $|I|=n-k$, and $y_i\in\{0,1\}$, $i\in I$ are fixed values. 
Without loss of generality let $I=\{k+1,\ldots, n\}$. 
Consider the following choice of parameters. 
Define $W_{i,j} = \delta_{i,j}$ for $i,j\leq k$ and zero else, define $B_j=-\tfrac{1}{2}$ for $j\leq k$ and $B_j=y_j - \tfrac{1}{2}$ else, and $C=-\tfrac{1}{2}\mathds{1}$. 
Then $\sgn(W v +C) = \sgn(v_1-\tfrac{1}{2},\ldots, v_k-\tfrac{1}{2},  -\tfrac{1}{2},\ldots, -\tfrac{1}{2})$, so that $h_v=(v_1,\ldots, v_k, 0,\ldots, 0)$. 
Furthermore, $\sgn(h_v^\top W + B^\top) = \sgn(v_1-\tfrac{1}{2}, \ldots, v_k -\tfrac{1}{2}, y_{k+1}-\tfrac{1}{2},\ldots, y_n-\tfrac{1}{2})$. 
\end{example} 
Later we will study the ability of RBMs to reconstruct more complicated sets of binary inputs, and how the sets of reconstructible inputs relate to the visible probability distributions represented by the model.

\subsection{Modes}\label{section:modes}

We will characterize the ability of RBMs and mixtures of product distributions to represent distributions with many strong modes, in order to draw a distinction between them. 
As an interesting side remark, note that similar questions, about the number of modes of mixtures of multivariate normal distributions, have been posed in~\cite{Ray}, and that the maximal number of modes realizable by mixtures of $k$ normal distributions on $\mathbb{R}^n$ is unknown.  

\begin{definition}\label{strongmodedef}
Let $p$ be a probability distribution on a finite set $\Xcal$ of length-$n$ vectors. 
A vector $x\in\Xcal$ is a {\em mode} of $p$ if $p(x) > p(y)$ for all $y\in\Xcal$ with $d_H(y,x)=1$, and a {\em strong mode} if $p(x) > \sum_{y\in\Xcal:d_H(y,x)=1}p(y)$. 
Here $d_{H}(y,x):=|\{i\in[n]\colon y_i\neq x_i\}|$ is the Hamming distance between $y$ and $x$. 
\end{definition}

The modes of a distribution are the Hamming-locally most likely events in the space of possible events. 
Modes are closely related to the support sets and boundaries of statistical models, 
which have been studied especially for hierarchical and graphical models without hidden variables~\cite{geiger:2006,kahle:2009,RKA10:SupportSetsOrMat}. 

We write $\mathcal{G}_{n,m}$ (and $\mathcal{H}_{n,m}$) for the set of distributions in $\Pcal_n$ which have at least $m$ modes (strong modes). For any set $\Ccal\subset\{0,1\}^n$ of vectors with Hamming distance at least $2$ from each other, we write $\Gcal_\Ccal$ (and $\Hcal_\Ccal$) for the set of distributions which have modes (strong modes) $\Ccal$. 
The closures $\overline{\Gcal_\Ccal}$  (and $\overline{\Hcal_\Ccal}$) are convex polytopes inscribed in the probability simplex $\Pcal_n$. 
The sets of modes that are not realizable by a probability model give a full dimensional polyhedral approximation of the model's complement. 
See Figure~\ref{fig2} for an example. 
We will focus most of our consideration on strong modes. These are easier to study than modes, because they are described by fewer inequalities. 

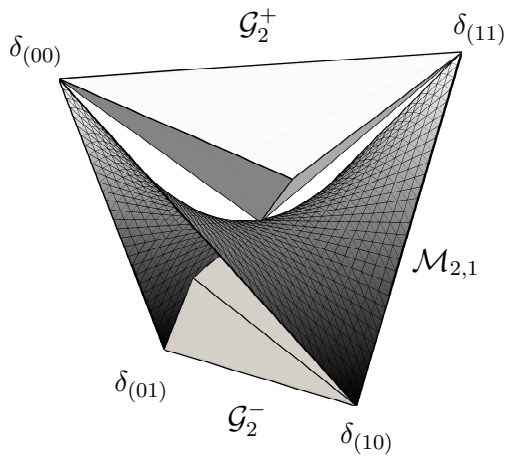
\begin{figure}
\begin{center}
\setlength{\unitlength}{5cm}
\begin{picture}(1.8,1.18)(0.1,0.075)
\put(0,0){\includegraphics[clip=true,trim=2.7cm 8.5cm 2cm 9.8cm, width=10cm]{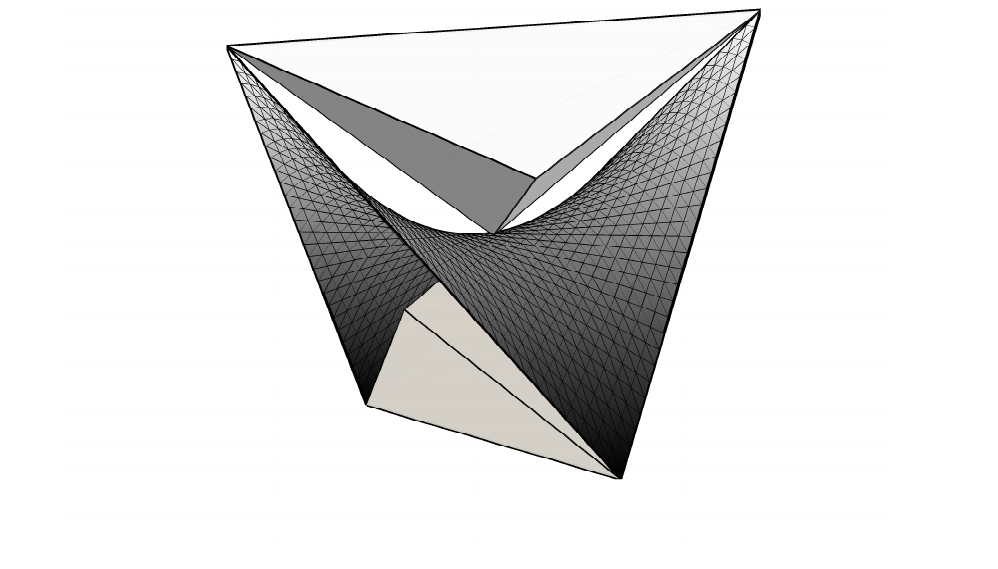}}
\put(.93,1.17){$\mathcal{G}_{2}^+$}
\put(.9,.12){$\mathcal{G}_{2}^-$}
\put(1.39,.54){$\Mcal_{2,1}$}
\put(0.32,1.11){\small $\delta_{(0 0)}$}
\put(1.5,1.17){\small $\delta_{(1 1)}$}
\put(0.6,.22){\small $\delta_{(0 1)}$}
\put(1.2,.08){\small $\delta_{(1 0)}$}
\end{picture}
\end{center}
\caption[Bimodal mixtures of product distributions of two binary variables]{
The $3$-dimensional simplex of probability distributions on $\{0,1\}^2$ (a tetrahedron with vertices corresponding to the outcomes $(00),  (01), (10), (11)$), with three sets of probability distributions depicted. The dark curved surface is the 2-dimensional manifold $\Mcal_{2,1}$ of product distributions of two binary variables. 
The angular regions at the top and bottom are the polyhedra $\Gcal_2^+$ and $\Gcal_2^-$ of distributions with two modes. 
An interactive $3$-D graphic object is available at~\url{http://personal-homepages.mis.mpg.de/montufar/surface.pdf}. 
}\label{fig2}
\end{figure}

The minimum Hamming distance of a set $\Ccal\subseteq\Xcal$ is defined as $d_H(\Ccal):=\min\{d_H(x,y)\colon x\neq y\text{ and } x,y\in \Ccal\}$. Since any two modes have at least Hamming distance two from each other, a distribution on $\{0,1\}^n$ has at most $2^{n-1}$ modes. There are exactly two subsets of $\{0,1\}^n$ with cardinality $2^{n-1}$ and minimum distance two. These are the sets of binary strings with an even, respectively odd, number of entries equal to one: 
\begin{eqnarray*}
Z_{+,n}&:=&\Big\{(x_1,\ldots,x_n)\in\{0,1\}^n\colon \sum_{i\in[n]} x_i \text{ is even} \Big\}; \\ 
Z_{-,n}&:=&\Big\{(x_1,\ldots,x_n)\in\{0,1\}^n\colon \sum_{i\in[n]} x_i \text{ is odd} \Big\}.
\end{eqnarray*}  
We write $\Gcal_{n,2^{n-1}}=\Gcal_{Z_{+,n}}\cup\Gcal_{Z_{-,n}}$, or $\Gcal_n=\Gcal_{n}^+\cup\Gcal_n^-$ for short, and similarly $\Hcal_n=\Hcal_n^+\cup\Hcal_n^-$. 
Figure~\ref{fig2} illustrates the set $\Gcal_2 =\Gcal_2^+\cup \Gcal_2^-\subset \Pcal_2$ (the set of distributions on $\{0,1\}^2$ with two modes), and the two-dimensional manifold $\Mcal_{2,1}\subset\Pcal_2$ (the set of product distributions on $\{0,1\}^2$). 
We see that $\Pcal_2$ contains $6$ disjoint sets congruent to $\Gcal_2^+$ whose union's closure equals $\Pcal_2$. 
Hence the Lebesgue volume satisfies $\vol(\Gcal_2^+)/\vol(\Pcal_2)=1/6$. 
The case of three bits is as follows.

\begin{example}
The subset $\Gcal_{3}^+\subset \Pcal_3$ of distributions on $\{0,1\}^3$ with four modes $Z_{+,3}$ is the intersection of $\Pcal_3$ and $12$ open half-spaces defined by $ p(x) > p(y)$ for all $y \text{ with } d_H(x,y)=1$ for all $x\in Z_{+,3}$. 
The closure of this set is a convex polytope $\overline{\Gcal^+_{3}}$ with $19$ vertices. 
The vertices are uniform distributions on subsets of $\{0,1\}^3$ that can be covered by three disjoint cylinder subsets of $\{0,1\}^3$. 
The list of vertices and vertex-facet incidences are provided in Tables~\ref{table:vertices} and~\ref{table:facets}, in the appendix. 
The Lebesgue volume of this polytope can be computed (e.g., using the software \texttt{Polymake}~\cite{polymake}): $\vol(\mathcal{G}_3^+)/\vol(\mathcal{P}_3) = 1/56 = 0.017\overline{857142}$. 

The subset $\Hcal_{3}^+\subset \Gcal_3^+$ of distributions with four strong modes $Z_{+,3}$ is the intersection of $\Pcal_3$ and the $4$ open half-spaces defined by $p(x) > \sum_{y:d_H(x,y)=1}p(y)$ for all $x\in Z_{+,3}$. 
We will discuss it in Section~\ref{sec:polyapprRBM32}. 
\end{example}

\subsection{Modes of mixtures of products}\label{section1}

In this section we characterize the sets of modes and strong modes that can appear in mixtures of product distributions, and show how these can be used to obtain a polyhedral approximation of the set of probability distributions representable in such models.

\begin{problem}\label{questionsec}\label{questionprim}
What is the smallest $k\in\mathbb{N}$ for which $\Mcal_{n,k}$ contains a distribution with $l$ (strong) modes? 
\end{problem}

A mixture of $k$ unimodal discrete probability distributions has at most $k$ strong modes. 
For mixtures of products we have the following. 

\begin{theorem}\label{sobrecomplmix}
Let $\Xcal_1,\ldots, \Xcal_n$ be finite sets. 
Let $\Mcal$ be the set of all possible mixtures of $k$ product distributions of $n$ variables with state spaces $\Xcal_i$, $i\in[n]$. 
If $p\in\Mcal$ has strong modes $\Ccal$, then every $c\in\Ccal$ is the mode of one mixture component of $p$. 
The sets of strong modes of distributions within $\Mcal$ are exactly the sets of strings in $\Xcal_1\times\cdots\times\Xcal_n$ of minimum Hamming distance at least two and cardinality at most $k$. 
\end{theorem}

\begin{proof}
A product distribution $q$ has at most one mode. This follows from the fact that the value of $q(x_1,\ldots, x_n) = q_1(x_1)\cdots q_n(x_n)$ is either maximal, or can be increased by changing only one entry of $x$.  If $q^{(j)}$, $j\in[k]$, are product distributions and $x$ is not a mode of any $q^{(j)}$, then $\sum_{y:d_H(y,x)=1}\sum_{j\in[k]}\alpha_j q^{(j)}(y) \geq \sum_{j\in[k]}\alpha_j q^{(j)}(x)$ for any $\alpha_j\geq 0$. In turn, $x$ is not a strong mode of any mixture of the $q^{(j)}$. 
On the other hand, the set of product distributions contains every point measure $\delta_y$, since the latter can be written as $\delta_y(x) = q_1(x_1)\cdots q_n(x_n)$ with $q_i(y_i)=1$ for all $i\in[n]$. Hence the mixture of products model ${\Mcal}$ contains every distribution of the form $\sum_{y\in\Ccal}\frac{1}{|\Ccal|}\delta_y$ for any $\Ccal\subseteq\Xcal$ with $|\Ccal|\leq k$. 
\end{proof}

By the previous theorem, a mixture of $k$ product distributions can have at most $k$ strong modes. 
Nevertheless, a mixture of $k$ product distributions can have more than $k$ modes. Here is an example: 

\begin{example}\label{firstcorollary} 
	\newcommand{\sm}[2]{\left(\begin{smallmatrix}#1\\ #2\end{smallmatrix}\right)}
	\newcommand{\smo}[4]{\!\!\begin{smallmatrix}#1 & #2\\#3 & #4\end{smallmatrix}\!\!}
	\newcommand{\pz}{0}
The mixture 
$p=\tfrac{1}{2} p^{(1)} + \tfrac{1}{2} p^{(2)}$ 
of the following two product distributions of four binary variables 
has three modes $(0000)$, $(1100)$, and $(1111)$:  
\begin{flalign*}
p^{(1)} 
= \tfrac{1}{5}\sm{3}{2} \otimes \tfrac{1}{5}\sm{3}{2} \otimes \sm{1}{0}\otimes \sm{1}{0} 
&= \tfrac{1}{25} \left[ \begin{array}{c|c} \smo{9}{6}{6}{4} & \smo{\pz}{\pz}{\pz}{\pz} \\ \hline \smo{\pz}{\pz}{\pz}{\pz} & \smo{\pz}{\pz}{\pz}{\pz} \end{array} \right]; \\ 
p^{(2)} 
= \sm{0}{1} \otimes \sm{0}{1} \otimes \tfrac{1}{5}\sm{2}{3} \otimes \tfrac{1}{5}\sm{2}{3} 
&= \tfrac{1}{25} \left[ \begin{array}{c|c} \smo{\pz}{\pz}{\pz}{4} & \smo{\pz}{\pz}{\pz}{6} \\ \hline \smo{\pz}{\pz}{\pz}{6} & \smo{\pz}{\pz}{\pz}{9}  \end{array} \right].
\end{flalign*}
\end{example}

Theorem~\ref{sobrecomplmix} shows that the set $\mathcal{H}_{n,k+1}$ is in the complement of $\Mcal_{n,k}$ for all $k$. 
We can triangulate $\Hcal_{n,k+1}$ and thereby lower bound the (Lebesgue) volume of the complement $\Pcal_n\setminus \Mcal_{n,k}\supseteq \Hcal_{n,k+1}$. 
A rough estimate is:
 
\begin{proposition}\label{volmodesset}
Let $k<2^{n-1}$. The volume of $\Hcal_{n,k+1}$ satisfies 
$\vol(\mathcal{H}_{n,k+1})/\vol(\Pcal_n) \geq  2^{-(k+1) n}K(k+1)$,  
where $K(k+1)= 2^{k+1}$ if $k+1\leq 2^s<\frac{2^n}{n}$ for some $s\in\mathbb{N}$, and $K(k+1)=2$ otherwise. 
\end{proposition}
\begin{proof}
Let $\Ycal\subseteq\Xcal:=\{0,1\}^n$. 
Let $\Pcal(\Y)$ be the simplex of probability distributions with support in $\Ycal$. 
This is a regular $(|\Y|-1)$-simplex in $\mathbb{R}^{|\X|}$ with edge-length $\sqrt{2}$. 
Let $\Hcal(\Y)$ denote the set of distributions on $\Xcal$ with strong modes $\Y$, such that $\Hcal(\Y)=\cap_{y\in\Y}\Hcal(y)$. Let $B_1(y)\subseteq\Xcal$ denote the radius-$1$ Hamming ball centered at $y$. 
The set $\Pcal_y(B_{1}(y)):=\{p\in\Pcal(B_{1}(y))\colon p(y)\geq p( B_1(y)\setminus\{y\})\}$ is a regular $n$-simplex with edge-length $\frac{\sqrt{2}}{2}$ and  vertices $\{\frac{1}{2}(\delta_y+\delta_{\hat y})\}_{d_H(\hat y,y)\leq 1}$. 
The volume of a regular $N$-simplex with edge-length $l$ is $\frac{\sqrt{N+1}}{N!\sqrt{2}^N} l^N$. 
The set $\Hcal(y)$ is the convex hull of $\Pcal_y(B_{1}(y))$ and $\Pcal(\X\setminus B_{1}(y))$. 
Its volume satisfies 
$$\operatorname{vol}(\Hcal(y))/\vol(\Pcal_n) = \vol(\Pcal_y(B_{1}(y)))/\vol(\Pcal(B_{1}(y)) ) = 2^{-n}.$$ 

If $\Y$ has minimum distance $3$ or more, then the radius-$1$ Hamming balls $B_1(y)$, $y\in\Ycal$, are disjoint, and 
$$\vol(\Hcal(\Y)) / \vol(\Pcal_n)=\prod_{y\in\Ycal}\left( \vol(\Pcal_y(B_{1}(y)))/\vol(\Pcal(B_{1}(y)) ) \right)=  2^{-|\Y|n}.$$ 

If $\Ycal$ has minimum distance $2$ instead of $3$, then the volume of $\Hcal(\Y)$ can only go up. To see this, consider a pair $y,y'$ of Hamming distance $2$ and let $\Pcal_{y,y'}(B_1(y)\cup B_1(y')):=\{p\in\Pcal(B_1(y) \cup B_1(y') )\colon p(y)\geq p( B_1(y)\setminus\{y\}), p(y')\geq p( B_1(y')\setminus\{y'\})  \}$ denote the closure of the set of distributions with support in $B_1(y)\cup B_1(y')$ and strong modes $y$ and $y'$. 
Then 
$$\vol( \Pcal_{y,y'}(B_1(y)\cup B_1(y') ) ) / \vol( \Pcal(B_1(y)\cup B_1(y') ) ) > 2^{-2n}. $$ 
This is because, for any $z$ with Hamming distance one to both $y$ and $y'$, 
the inequality $p(z)< 1/2$ implied by $\Hcal(y)$ is also implied by $\Hcal(y')$ and hence not each pair of neighbors of the form $(y,x)$ or $(y',x)$ translates to a factor $1/2$ in the volume. 

The number $K(k+1)$ is a lower bound on the number of disjoint sets $\Hcal(\Y)$ with $|\Y|=k+1$. 
By the Gilbert-Varshamov bound, if $k+1\leq 2^s$ for some integer $s$ with $2^s < \frac{2^n}{n}$, then there is a set $\Y\subset\X$ of cardinality $|\Y|=k+1$ and minimum distance~$3$. 
Let $\Y'=(\Y\setminus\{y\})\cup\{y \oplus_2 e_1\}$ (flip one coordinate of one element of $\Y$), such that $\Hcal(\Y)\cap\Hcal(\Y')=\emptyset$. Since $\Y$ has $(k+1)$ elements, there are $2^{k+1}$ disjoint sets of this form. 
For a general $\Y\subset\Xcal$ of cardinality $|\Ycal| = k+1\leq2^{n-1}$ and minimum distance~$2$, 
the set $\Y'\oplus_2 e_1$ also has minimum distance $2$ and cardinality $|\Ycal'|=k+1$. These two sets satisfy $\Hcal(\Y)\cap\Hcal(\Y')=\emptyset$. 
\end{proof}

\subsubsection[Polyhedral approximation of the full-dimensional model M33]{Polyhedral approximation of the full-dimensional model ${\Mcal_{3,3}}$}\label{indepthreemodes} 

Theorem~\ref{sobrecomplmix} shows that any $p\in\Mcal_{3,3}$ has at most three strong modes. 
We can show that this is true for modes too. 

\begin{proposition}\label{theorem_modes_three_mixture}
The mixture model of three product distributions on $\{0,1\}^3$ cannot realize distributions with four modes:    
${\Mcal_{3,3}}\cap \mathcal{G}_{3}=\emptyset$.    
\end{proposition}

\begin{proof}
Assume that $\Mcal_{3,3}\cap\Gcal^+_3\neq\emptyset$. 
By the Lemma~\ref{lemmaparamixtthree} given below, there are factors  $(p^{i}_1, p^{i}_2, p^{i}_3)\in(\Pcal_1)^3, i=1,2,3$ such that $\conv\{q^i:=p^{i}_2p^{i}_3\}_{i=1,2,3}$ intersects $\Gcal^+_2$ and $\Gcal^-_2$. 
Hence $\operatorname{conv}\{q^1,q^2\}$ intersects $\Gcal^+_2$ and $\operatorname{conv}\{q^2,q^3\}$ intersects $\Gcal^-_2$ (for some enumeration of $q^1,q^2,q^3$). The mixture of $q^2$ and $q^3$ intersects $\Gcal^-_2$ only if $(01)$ and  $(10)$ are the unique maxima of $q^2$ and $q^3$. Similarly, $\operatorname{conv}\{q^1,q^2\}$ intersects $\Gcal^+_2$ only if $(11)$ and $(00)$ are the unique maxima of $q^1$ and $q^2$; a contradiction. 
\end{proof}

The proof of Proposition~\ref{theorem_modes_three_mixture} uses the following lemma, which relates the number of modes realizable by $\Mcal_{n,k}$ to the number of modes simultaneously realizable on subsets of variables. 

\begin{lemma}\label{lemmaparamixtthree}
Let  $n,k\in\mathbb{N}$ and $n\geq 2$. 
Let $p=\sum_{i\in[k]}\lambda_i\prod_{j\in[n]} p^{i}_j\in\Mcal_{n,k}$, with $\lambda_i\geq 0$ and $p^{i}_j\in \Pcal_1$ for all $(i,j)\in[k]\times[n]$.  
If $p$ has $2^{n-1}$ modes, then for any subset of variables $I\subsetneq[n]$, $|I|=m$, the convex hull of the product distributions $\{\prod_{j\in I} p^{i}_j\}_{i\in[k]}\subset\Pcal_m$ intersects both $\Gcal_m^+$ and  $\Gcal_m^-$. 
\end{lemma}
\begin{proof}
We show the  special case with $I=\{1,\ldots,n-1\}$. The proof of the general case is a straightforward generalization. 
Any $q\in\Mcal_{n,k}$ has the following form: 
\begin{equation*}
q(x_1,x_2,\ldots,x_n)=\sum_{i=1}^k \lambda_i p^{i}_1(x_1)p^{i}_2(x_2)\cdots p^{i}_n(x_n),
\end{equation*} 
for all $(x_1,x_2,\ldots,x_n)\in\{0,1\}^n$, 
where $\sum_{i=1}^k\lambda_i=1$, $\lambda_i\geq 0$ and $p^{i}_j\in\Pcal_1$. 
For the fixed value $x_1=0$ this is a mixture  of $k$ products with $(n-1)$ variables, multiplied by a positive constant:
\begin{equation*}
q(x_1=0,x_2,\ldots,x_n) =  c_0  {\sum_{i=1}^k \lambda_{0,i} p^{i}_2(x_2)\cdots p^{i}_n(x_n)} , 
\end{equation*}
where $\sum_{i=1}^k \lambda_{0,i}=1, \lambda_{0,i}\geq 0$ with 
$\lambda_{0,i}= \frac{\lambda_{i} p^{i}_1(x_1=0)}{c_0}$ and  $c_0=\sum_{i=1}^k\lambda_{0,i} p^{i}_1(x_1=0)$. 
A similar observation can be made for the fixed value $x_1=1$. 
If the distribution $q$ is contained in $\mathcal{G}_{n}^+$, then $q(x_1=0,x_2,\ldots,x_n) \in \mathcal{G}_{n-1}^+$ and $q(x_1=1,x_2,\ldots,x_n)\in \mathcal{G}_{n-1}^-$, since $q\in \Gcal^+_n$. 
These two conditional distributions are mixtures of the same $k$ product distributions $\{p^{i}_2\cdots p^{i}_n\}_{i\in[k]}$, even though they may have different mixture weights. 
\end{proof}

\begin{remark}
The sets $\overline{\Gcal_3^+}$ and $\overline{\Gcal_3^-}$ are intersections of  half-spaces that contain the uniform distribution. 
Although $\Mcal_{3,3}$ is full dimensional in $\Pcal_3$, by Proposition~\ref{theorem_modes_three_mixture} the complement of $\Mcal_{3,3}$ contains points arbitrarily close to the uniform distribution! 
\end{remark}

\subsection{Modes of RBMs}\label{section2}

In the following we characterize the sets of modes and strong modes that can appear in RBM-distributions. 
In analogy to the problem posed in the last section, we ask: 
\begin{problem}
What is the smallest $m\in\mathbb{N}$ for which $\RBM_{n,m}$ contains a distribution with $l$ (strong) modes? 
\end{problem}

In particular, what is the smallest $m$ for which the model $\RBM_{n,m}$ can represent the parity function? 
By Theorem~\ref{sobrecomplmix} and $\RBM_{n,m}\subseteq\Mcal_{n,2^m}$,  any $p\in\RBM_{n,m}$ has at most $\min\{2^m,2^{n-1}\}$ strong modes. We will see that this bound is not always tight, but often. 

In special cases, the analysis of mixtures of products (Section \ref{section1}) is sufficient to make statements about RBMs, for example: 
The model $\RBM_{4,2}$ is contained in $\Mcal_{4,4}$ and has co-dimension one in $\Pcal_4$. 
Its {\em algebraic implicitization}\index{algebraic implicitization} was studied in~\cite{CuetoChallenge2010}, i.e., its description as the set of zeros of a collection of polynomials. It was found to be the zero locus of a polynomial of degree $110$ with as many as $5.5$ trillion monomials. 
By Theorem~\ref{sobrecomplmix}, $\Mcal_{4,4}\cap\Hcal_{4}=\emptyset$ and so $\RBM_{4,2}\cap\Hcal_{4}=\emptyset$.  
Using Proposition~\ref{theorem_modes_three_mixture} and Lemma~\ref{lemmaparamixtthree} one can show: 
\begin{proposition}\label{cororbm42}
The models $\Mcal_{4,4}$ and $\RBM_{4,2}$ cannot realize probability distributions with $8$ modes: 
 $\Mcal_{4,4}\cap\,\Gcal_{4}=\emptyset$ and $\RBM_{4,2}\cap\,\Gcal_{4}=\emptyset$. 
\end{proposition}

We note that the model $\RBM_{n,m}$ contains any distribution with support of cardinality $\min\{m+1,2^n\}$~\cite[Theorem~1]{Montufar2011}. 
Therefore, it contains some distributions with $\min\{m+1,2^{n-1}\}$ strong modes. For example, $\RBM_{n,m}$ contains any uniform distribution on a set of cardinality $\min\{m+1,2^{n-1}\}$ and minimum Hamming distance $2$. 
In particular, whenever $\Mcal_{n,k+1}$ contains a distribution with strong modes $\Ccal$, then also $\RBM_{n,k}$ contains a distribution with strong modes $\Ccal$. 

Note also that, since the model $\RBM_{n,m}$ is symmetric under relabeling of any of its variables, there is an RBM distribution with strong modes $\Ccal$ iff  there is one with strong modes $\Ccal \oplus_2x =\{c+x\mod(2)\colon c\in\Ccal\}$ for any $x\in\{0,1\}^n$. 

In general, characterizing the sets of modes realizable by RBMs is a more complex problem than it was for mixtures of product distributions, and will necessitate developing characterizations in terms of point configurations called zonosets (Definition \ref{def:zonoset}) and hyperplane arrangements, or in terms of linear threshold functions. 
We elaborate these notions and characterizations in the next 3 sections. 

\subsection{Zonosets and hyperplane arrangements}\label{section:zonosets}

\begin{definition}\label{def:zonoset}
Let $m\geq0$,  $n>0$, $W_i\in\R^n$ for all $i\in[m]$, and $B\in\mathbb{R}^n$. 
The multiset $\Zcal=\{\sum_{i\in I}W_i + B\}_{I\subseteq[m]}$ is called an {\em $m$-zonoset}. 
\end{definition}
The convex hull of a zonoset is a zonotope, a well known object in the literature of polytopes. 
Zonotopes can be identified with hyperplane arrangements and oriented matroids~\cite{bjoerner1999oriented,ziegler1995lectures}.

Given a  sign vector $s\in\{-,+\}^n$, the $s$-orthant of $\R^n$, denoted $\R_s^n$, consists of all vectors $x \in \R^n$ with $\sgn(x) =s$. 
We say that an orthant has even (odd) parity if its sign vector has an even (odd) number of  $+$. 
The sets of strong modes of RBMs can be described in terms of zonosets as follows. 

\begin{theorem}\label{zonotopecondpropo}
Let $\Ccal\subset\{0,1\}^n$ have minimum Hamming distance at least two. 
\begin{itemize}
\item 
If the model $\RBM_{n,m}$ contains a distribution with strong modes $\Ccal$ (i.e., $\RBM_{n,m}\cap\Hcal_\Ccal\neq \emptyset$), 
or $\Ccal$ has cardinality $2^m$ and is perfectly reconstructible by $\RBM_{n,m}$, 
then there is an $m$-zonoset with a point in each $\Ccal$-orthant of $\R^n$. 
\item 
If there is an $m$-zonoset intersecting exactly the $\Ccal$-orthants of $\R^n$ at points of equal $\ell_1$-norm, then 
$\RBM_{n,m}\cap\mathcal{H}_\Ccal\neq\emptyset$ and, furthermore, $\Ccal$ is perfectly reconstructible. 
\end{itemize}
\end{theorem}

\begin{proof}
Assuming that $p \in \RBM_{n,m}\cap \mathcal{H}_{\Ccal}$, for each $x\in\Ccal$ there is an $h\in\{0,1\}^m$ for which $p(\cdot|h)$ is uniquely maximized by $x$ (Theorem~\ref{sobrecomplmix} and $\RBM_{n,m} \subset \mathcal{M}_{n,2^m}$). This is also true if $\Ccal$ is perfectly reconstructible. 
In this case, $(h^\top W  +  B^\top) x   >  (h^\top W  + B^\top) v $ for all $ v\neq x$, 
and, equivalently, $\operatorname{sgn}(h^\top W + B^\top ) =  \operatorname{sgn} ( x -\frac{1}{2}(1,\ldots, 1)^\top)$. The existence of such $W$ and $B$ is equivalent to the existence of a zonoset with a point in each $\Ccal$-orthant of $\R^n$. 

Assume now that $W,B$ can be chosen such that all vectors $h^\top W +B^\top$ have the same $\ell_1$ norm, equal to $K$. 
We have $\tfrac{1}{2} K=\tfrac{1}{2}\|h^\top W +B^\top \|_1=(h^\top W + B^\top )(x_h-\tfrac{1}{2}(1,\ldots, 1)^\top)=(h^\top W+B^\top)x_h+h^\top C  -\tfrac{1}{2} B^\top(1,\ldots,1)^\top$, where $C=-\tfrac{1}{2} W (1,\ldots,1)^\top$, 
for some $x_h\in\Ccal$ for all $h\in\{0,1\}^m$.  
The RBM with parameters $\alpha W, \alpha B$, $C=-\alpha W \tfrac{1}{2}(1,\ldots,1)^\top$, and  $\alpha\to\infty$ produces $\frac{1}{2^m}\sum_{h\in\{0,1\}^m}\delta_{x_h} \in \Hcal_\Ccal$ as its visible distribution. 
This also implies that $\Ccal$ is perfectly reconstructible. 
\end{proof}

\begin{remark}\label{remarkzono}
The first part of Theorem~\ref{zonotopecondpropo} remains true if $\Hcal_\Ccal$ is extended to the set of distributions for which any $\Mcal_{n,2^m}$-decomposition has a mixture component with  mode $c$, for every $c\in\Ccal$. 
\end{remark}

\begin{definition}\label{def:hyperplanearrangement}
A {\em hyperplane arrangement} $\Acal$ in $\R^n$ is a finite set of (affine) hyperplanes $\{H_i\}_{i\in[k]}$ in $\R^n$. 
Choosing an orientation for each hyperplane, each vector $x\in\R^n$ receives a sign vector $\sgn_\Acal(x)\in \{-,0,+\}^k$, where $(\sgn_\Acal(x))_i$ indicates whether $x$ lies on the negative side, inside, or on the positive side of $H_i$. 
The set of all vectors in $\R^n$ with the same sign vector is called a {\em cell} of $\Acal$. 
\end{definition}
A necessary condition for the existence of an $m$-zonoset intersecting all $\Ccal$-orthants of $\R^n$ is that the number of orthants of $\mathbb{R}^n$ that are  intersected by an $m$-dimensional affine space is at least $|\Ccal|$. 
The maximal number of orthants intersected by a $d$-dimensional linear subspace of $\R^n$, denoted $\Cn(n,d)$,  was derived in~\cite{schlaefli1953}. 
It is not difficult to derive the corresponding number for a $d$-dimensional affine subspace, denoted $\Ca(n,d)$, as well:
\begin{equation}
\Cn(n,d) =2\sum_{i=0}^{d-1}{n-1\choose i}\qquad\text{and}\qquad \Ca(n,d) =\sum_{i=0}^{d}{n\choose i}. \label{dichotomiehomogeneous}
\end{equation} 
Cover~\cite{Cover} shows that $\Cn(n,d)$ is also the number of partitions of an $n$-point set in general position in $\R^d$  by central hyperplanes (hyperplanes through the origin). A set of vectors in $\R^d$ is in general position if any $d$ or less are linearly independent. 
Dually, $\Ca(n,d)$ can be seen as the number of cells of a real $d$-dimensional arrangement of $n$ hyperplanes in general position~\cite{orlik1992,schlaefli1901theorie}. 

In particular,  there are affine hyperplanes of $\R^n$ intersecting all but one orthants. 
Figure~\ref{inferenceregions} (right) is an example  showing the intersection of a $2$-dimensional affine subspace of $\R^3$ and $7$ orthants; four of odd parity and three of even parity. 
This does not imply, however, that every collection of $2^m$ even, or odd, orthants can be intersected by an $m$-zonoset. For example: 

\begin{proposition}\label{oddnot}
If $n$ is an odd natural number larger than one, then there is no $(n-1)$-zonoset with a point in every  even, or every odd, orthant of $\R^n$. 
\end{proposition}

\begin{proof}
Let $\Zcal$ be a candidate zonoset; $\Zcal$ has $(n - 1)$ generators, so it lies in an affine hyperplane $H$ of $\mathbb{R}^n$. Let $\eta$ be a normal vector to $H$. 
Assume first that $0 \in H$. 
All vectors in the orthants $\operatorname{sgn}(\eta)$ and $-\!\operatorname{sgn}(\eta)$ lie outside $H$ (where we may assign arbitrary sign to zero entries of $\eta$). 
This follows from Stiemke's theorem, see, e.g.,~\cite{Flatto}. 
The two orthants have opposite sign vectors and $n$ is odd, so one orthant is even and the other odd. 
Hence at least one even and one odd orthants do not intersect $\Zcal$. 

Consider now an affine subspace $H$, and assume it intersects all even orthants. By eq.~\eqref{dichotomiehomogeneous} $\dim(H) \geq  n - 1$, so $H$ is a hyperplane. 
Assume without loss of generality that a normal vector to $H$ has only negative entries. 
Then $H\cap \R^n_{(-\cdots-)}$ is an $(n - 1)$-dimensional simplex containing a point of $\mathcal{Z}$. 
This can be inferred from the number of bounded cells in a $d$-dimensional arrangement of $n$ hyperplanes in general position, $b(n,d)={n-1\choose d}$~\cite[Proposition~2.4]{Stanley04}. 
The orthant $\R^n_{(-\cdots-)}$ is separated by $(n-1)$ coordinate hyperplanes  from the orthant $\R^n_{s_i}$ with sign $s_i = (+\cdots+\underset{i}{-}+\cdots+)$ for any $i\in[n]$. 
Since $n$ is odd and larger than one, $(n - 1)>0$ is even. 
Since $\mathcal{Z}$ intersects $H\cap \R^n_{s_i}$ for all $i\in[n]$, also $H\cap \R^n_{(-\cdots-)} \subset \conv(\mathcal{Z})$ (details in Lemma~\ref{facetcone}). 
On the other hand, the $(n-1)$-generated zonotope of dimension $(n - 1)$ is combinatorially equivalent to the $(n - 1)$-cube, and no point in its zonoset is contained in the convex hull of any other points. 
\end{proof}

We used the following lemma in the proof of Proposition~\ref{oddnot}. 

\begin{lemma} \label{facetcone}
Let $P$ be a polytope with vertex set $V$. 
Let $\{H_i\}_{i=1}^r$ the supporting hyperplanes of the facets of $P$ incident to a vertex $v\in V$, and assume $P$ is contained in the intersection of closed half-spaces $\cap_i H_i^{+}$. 
If  $v'$ is any point in $\cap_i H_i^-$, then the polytope $\conv(\{v'\} \cup (V \setminus \{v\} ))$ contains $P$. 
\end{lemma}

\begin{proof}
The case $v'=v$ is trivial, so let $v'\neq v$. 
It is sufficient to show that $v$ is not a vertex of $Q:=\conv(\{v'\}\cup V)$, from which $v\in \conv(\{v'\} \cup (V \setminus \{v\} ))$ and $P \subseteq \conv(\{v'\} \cup (V \setminus \{v\} ))$ follows. 
The point $v'$ is a vertex of $Q$, because $v'\not\in P$. 
Consider first the case where  $v'$ is in the interior of $\cap H_i^-$, which is to say that $v'$ is not contained in any $ H_i^+$. 
If $v$ was a vertex of $Q$, then one $H_i$ would support a facet of $Q$ (otherwise $v'$ would be incident to all facets incident to $v$). This would contradict the fact that $v'\not\in H_i^+$. 
The general case $v'\in\cap H_i^-$ results from continuity.  
\end{proof}

Proposition~\ref{oddnot} allows us to describe some distributions that cannot be represented by RBMs. 
A code $\Ccal\subset\{0,1\}^n$ extends another code $\Ccal' \subset\{0,1\}^r$, $r\leq n$,  if restricting $\Ccal$ to some $r$ indices yields $\Ccal'$.

\begin{corollary}\label{corollaryoddnot}
If $m$ is an even non-zero natural number and $m < n$, then $\RBM_{n,m}\cap\Hcal_\Ccal=\emptyset$ for
any code $\Ccal\subset\{0,1\}^n$ extending $Z_{+,m+1}$ or $Z_{-,m+1}$. 
In particular, when $n$ is an odd natural number larger than one, $\RBM_{n,n-1}$ cannot represent any distribution with $2^{n-1}$ strong modes. 
\end{corollary}

\begin{proof}
If there is an $m$-zonoset with points in every $\Ccal$-orthant of $\R^n$ and there is a restriction of $\Ccal$ to $Z_{+,m+1}$ or $Z_{-,m+1}$, then there is  an $m$-zonoset contradicting Proposition~\ref{oddnot}. By Theorem~\ref{zonotopecondpropo}, $\RBM_{n,m}$ cannot represent distributions with strong modes $\Ccal$. 
\end{proof}

As a side remark, Corollary~\ref{corollaryoddnot} implies, in particular, that the graphical probability model on the bipartite graph $K_{n,m}$ (a fully observable version of the RBM model) does not contain in its closure any distribution supported on a set $\Y\subset\{0,1\}^{n+m}$ with 

\begin{equation*}
\{(x_{i_1},\ldots,x_{i_{m+1}})\in\{0,1\}^{m+1}\colon (x_1,\ldots, x_n,x_{n+1},\ldots,x_{n+m})\in \Y\} = Z_{\pm,m+1}
\end{equation*} 
for some $1\leq i_1 < \cdots < i_{m+1} \leq n$. 

In Section~\ref{section6} (Corollary~\ref{cor65}) we extend the statement of Corollary~\ref{corollaryoddnot} by showing that $\RBM_{6,5}$ cannot represent distributions with $2^{6-1}$ strong modes.

\subsubsection[Polyhedral approximation of the full-dimensional model RBM32]{Polyhedral approximation of the full-dimensional model ${\RBM_{3,2}}$} 
\label{sec:polyapprRBM32}

The model $\RBM_{3,2}$ is particularly interesting, because it is the smallest candidate of an RBM universal approximator on $\{0,1\}^3$ in terms of the number of mixture components of the mixtures of products that it represents, but it has less than $2^{n-1}-1$ hidden units, the upper bound for the number of hidden units of the smallest RBM universal approximator given in~\cite{Montufar2011}. 
Note that the  model $\RBM_{3,1}=\Mcal_{3,2}$ is readily full dimensional. 

By Corollary~\ref{corollaryoddnot}, $\RBM_{3,2}$  does not contain any distribution with four strong modes. 
We illustrate this explicitly: By Theorem~\ref{zonotopecondpropo}, if $\RBM_{3,2}\cap\,\mathcal{H}_3\neq\emptyset$, then 
\begin{equation}
 \operatorname{sgn}
 \begin{pmatrix}
           B\\
	W_1 + B\\
	W_2 + B\\
	W_1 + W_2 + B
 \end{pmatrix}
{=} \begin{pmatrix}
           + &-& -\\
	-&+&-\\
	-&-&+\\
	+&+&+
       \end{pmatrix} \label{equationsigns} 
\end{equation}
up to permutations of rows. But it is quickly verified that this equation cannot be satisfied. 

The set $\mathcal{H}_{3}$ is the disjoint union of $\mathcal{H}_{3}^+ =\mathcal{H}_{Z_{+,3}}$ and $\mathcal{H}_{3}^- =\mathcal{H}_{Z_{-,3}}$. 
The set $\overline{\mathcal{H}_{3}^+}$ is the $7$-dimensional simplex defined by the intersection of the $8$ half-spaces with inequalities $p(z) \geq \sum_{y:d_H(z,y)=1}p(y)$ for $z \in Z_{+,3}$ and $p(y)\geq 0$ for all $ y\in Z_{-,3}$. Its vertices are the uniform distributions on the following sets: 
\begin{gather*}
  \{000, 001, 011, 101\},   
 \{011, 101, 110, 111\},  
 \{000, 010, 011, 110\},   \{000, 100, 101, 110\}, \\  \{000\},\{011\},\{101\},\{110\}. 
\end{gather*}
Denote these distributions by $u_1,\ldots, u_8\in\Pcal_3\subset\R^8$. 
The volume of $\Hcal_3^+$ satisfies

\begin{equation*}
\operatorname{vol}(\mathcal{H}_{3}^+)/\operatorname{vol}(\mathcal{P}_{3}) =\det(u_1,\ldots, u_8) = \frac{1}{256} = 0.00390625. 
\end{equation*}
Since $\Hcal_3^+$ and $\Hcal_3^-$ are congruent, we obtain $\vol(\Hcal_3) / \vol(\Pcal_3) = \frac{1}{128} = 0.0078125$. 
This is a lower bound for $\vol(\Pcal_3\setminus\RBM_{3,2})/\vol(\Pcal_3)$. 

Now let us briefly discuss to what extent $\Hcal_3$ exhausts the complement of $\RBM_{3,2}$. 
The first four vertices of $\overline{\Hcal_3^+}$, $u_1,\ldots, u_4$, are mixtures of two point measures and one uniform distribution on a pair of Hamming distance one. 
The last four vertices, $u_5,\ldots, u_8$, are the point measures on $Z_{+,3}$. 
By \cite[Theorem~1]{MonRauh2011}, all these distributions are contained in $\RBM_{3,2}$ (by symmetry, the vertices of $\overline{\mathcal{H}_{3}^-}$ are also in $\RBM_{3,2}$). 
The distributions in the relative interiors of the edges between the first four vertices are not in $\RBM_{3,2}$. 
The relative interior of edges connecting one of the first four and one of the last four vertices are in $\RBM_{3,2}$ if they have support of cardinality four and are not if they have support of cardinality five. 
We conjecture that $\RBM_{3,2}\cap\Gcal_{3}=\emptyset$.

\subsection{Linear threshold codes}\label{section:LTC}

\begin{definition}
A {\em linear threshold function} (LTF) with $m$ inputs is a function 
\begin{equation*}
f\colon\{0,1\}^m\to\{-,+\};\quad  y\mapsto \sgn((\sum_{j\in[m]} w_j y_j ) + b), 
\end{equation*}
where $w\in\R^m$ is called {\em weight vector} and $b\in\R$ {\em bias}. 
A subset $\Ccal\subset\{0,1\}^m\subset\R^m$ is {\em linearly separable} 
iff there exists an LTF with $f(\Ccal)=+$ and $f(\{0,1\}^m\setminus \Ccal)=-$. 
For convenience we identify $-/+$ and $0/1$ vectors via $-\leftrightarrow0$ and $+\leftrightarrow1$. 
The opposite $\overline{x}$ of a binary vector $x$ is the vector given by inverting all entries of $x$. 
\end{definition}

LTFs are also known as {\em McCulloch-Pitts neurons} and have been studied in the context of feed-forward artificial neural networks. 
The problem of separating subsets of vertices of the $m$-dimensional hypercube by hyperplane arrangements (multi-label classification) has drawn much attention, see, e.g., \cite{Wenzel2000284}. 
It is known that the logarithm of the number of LTFs with $m$ inputs is asymptotically of order $m^2$, see~\cite{Zuev,Ojha}, but the exact number is only known for  $m \leq 9$, see~\cite{Winder,Muroga,oeis}. 
The study of LTFs simplifies when  $f(x_1,\ldots,x_m)=\overline{f}(\overline{x_1},\ldots,\overline{x_m})$ for all $x\in \{0,1\}^m$, in which case they are called {\em self-dual}. 
If an LTF has an equal number of positive and negative points, then  it separates every input from its opposite and is self-dual. 

\begin{definition}\label{defLTC}
A subset $\Ccal\subseteq\{0,1\}^n \cong\{-,+\}^n$ is an {\em $(n,m)$-linear threshold code} (LTC)  if there exist $n$ linear threshold functions $f_i\colon\{0,1\}^m\to\{0,1\}$, $i\in[n]$ with 
\begin{equation*}
\{(f_1(y),f_2(y),\ldots, f_n(y)) \in \{0,1\}^n\colon y\in\{0,1\}^m\}=\Ccal. 
\end{equation*} 
Equivalently, $\Ccal$ is an $(n,m)$-LTC if it is the image of the $\down$ inference function of $\RBM_{n,m}$ for some choice of the model parameters. 
If all $f_i$ can be chosen self-dual, then $\Ccal$ is called {\em homogeneous}. 
\end{definition}

In the following examples, an LTF with $m$ inputs is written as a list of the vertices of the $m$-cube with a bar on inputs with negative output and no bar on inputs with positive output. 
For notational convenience, each vertex $x=(x_1,\ldots, x_m)\in\{0,1\}^m$ is labeled by $1+\sum_{j\in[m]}2^{j-1}x_j$, which is the decimal representation of the binary vector plus one. 
For example, an LTF with two inputs, mapping $(00), (01)$ to $0$ and $(10), (11)$ to $1$, is written as $\overline{12}34$.  

\begin{example}
Let $n=3$ and $m=2$. There are only two ways to linearly separate the vertices of the unit square into sets of cardinality two: 
$12\overline{34}$  and $1\overline{2}3\overline{4}$. 
These are the only possible columns of a homogeneous LTC with two inputs (up to opposites). The code $Z_{\pm,3}$ is not a $(3,2)$-LTC, because 
it has three non-equivalent columns. 
This shows that there does not exist a $2$-zonoset with vertices in the four even, or odd, orthants of $\R^3$, and that  $\RBM_{3,2}$ does not contain any distributions with four strong modes. 

An alternative way of proving this is as follows.     
The Hamming distance between any two elements of $Z_{\pm,n}$ is even. 
If the distance of any two vertices of the square induced by an arrangement of three lines is even and non-zero, then each edge of the square is sliced at least twice, and in total at least $8$ edges are sliced (repetitions allowed). 
On the other hand, each line slices at most two edges of the square, and so three lines slice at most $6$ edges (repetitions allowed). 
\end{example}

\begin{figure}
\begin{center}
\setlength{\unitlength}{1.2cm}
\begin{picture}(4.5,4.95)(-0.2,0.45)
\put(-.2,0.1){\includegraphics[clip=true,trim=4.5cm 14.cm 9cm 4.6cm,width=5.5cm]{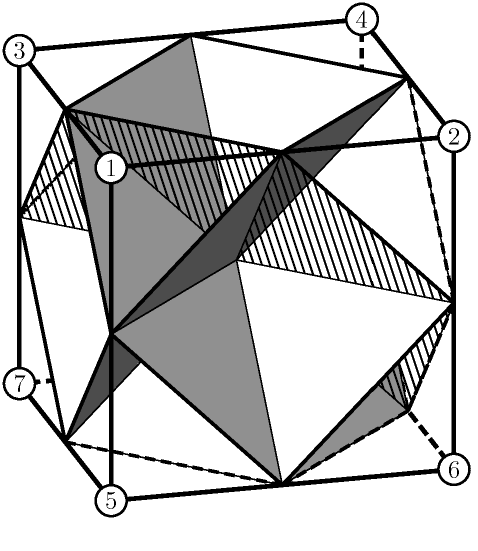}}
\end{picture}
\end{center}
\caption{
The four slicings of the 3-cube discussed in Example~\ref{exrbm43}.
}\label{slicescube}
\end{figure}

\begin{example}\label{exrbm43}
Let $n=4$ and $m=3$.  There are  $104$ ways to linearly separate  the vertices of the $3$-cube, see~\cite{Ojha}. 
A complete list appears in~\cite[Section~3.8]{bjoerner1999oriented}. 
The vertices of the $3$-cube are in the $Z_{+,4}$-cells of an arrangement of four hyperplanes corresponding to the $(4,3)$-LTC with following LTFs:  
\begin{equation*}
 123\overline{4}5\overline{678},\; {12}\overline{3}{4}\overline{5}{6}\overline{78},\; 1\overline{2}34\overline{56}7\overline{8},\;   1\overline{234}567\overline{8}.
 \end{equation*}
This arrangement corresponds to a $3$-zonoset with points in the $8$ even orthants of $\R^4$ (Theorem~\ref{thm:equivalences}). The zonoset can be realized as follows: 
\begin{equation}
\begin{minipage}[c]{.45\textwidth}
\begin{align}
{}\quad{}\quad w=&\left(\begin{array}{r r r r}
                         -1 & -1 & -1 &1 \\ -1 & -1 & 1& -1\\ -1 & 1 & -1& -1  
          \end{array}\right); \nonumber\\
b=&\frac{1}{2}\begin{pmatrix} 3& 1& 1& 1\end{pmatrix};\nonumber
\end{align}
\end{minipage}
\begin{minipage}[c]{.55\textwidth}
\begin{equation}
\Zcal=\frac{1}{2}\left(
\begin{array}{r r r r}3 & 1 & 1& 1 \\  1& 3& -1 & -1 \\  1 & -1  & 3 & -1\\   -1 & 1& 1& -3\\  1 & -1 & -1 & 3\\  -1 & 1& -3 & 1\\  -1 & -3 & 1& 1\\ -3 & -1 & -1 & -1
\end{array}
\right).
\quad{}\quad{}\label{weights43}
\end{equation}
\end{minipage}
\nonumber
\end{equation}
This choice of $w$ and $b$ corresponds to a central arrangement of four hyperplanes slicing each edge of the $3$-cube exactly twice, as shown in Figure~\ref{slicescube}. 
\end{example}

\newcommand{\p}{\,}
\begin{example} \label{n5example}
Let $n = 5$ and $m = 4$. 
There are three symmetry types of self-dual LTFs with four inputs, see~\cite{Muroga}. 
The following are representatives of the three types: 
\begin{gather*}
 \overline{1\p2 \p3\p  4 \p5\p  6 \p7\p8 } \p  9\p10 \p11\p  12 \p13\p  14 \p15\p16 \;;\\
\overline{ 1\p2 \p3\p  4 \p5\p6\p7} \p8\p  \overline{9} \p 10 \p 11 \p 12  \p 13 \p   14  \p 15 \p 16 \;;\\ 
\overline{ 1\p2\p3\p  4 \p5\p6} \p7\p8\p  \overline{9\p 10 } \p 11 \p 12  \p 13 \p   14 \p 15 \p 16  \;.
\end{gather*}
By Proposition~\ref{oddnot}, the code $Z_{\pm,5}$ cannot be realized by any $5$ LTFs, i.e., as a $(5,4)$-LTC, and $\RBM_{5,4}$ does not contain any distribution with $16$ strong modes.   
\end{example}

The following example presents a kind of binary code $\Ccal$ of cardinality $2^m$ with $\RBM_{n,m}\cap\Hcal_\Ccal=\emptyset$ which is not covered by Corollary~\ref{corollaryoddnot}.  
\begin{example}
Let $n = 5$ and $m=3$. 
Let $x',x''\in Z_{\pm,4}$ with $d_H(x',x'')=4$, and 
\begin{multline*}
\Ccal=\Big\{(x_1,\ldots,x_5)\colon (x_1,\ldots,x_{4})\in Z_{\pm,4},  
 x_5 =
\begin{cases}
1 &\text{ if } (x_1,\ldots,x_{4})\in\{x',x''\} \\
0 & \text{ otherwise}
\end{cases}\Big\}.
\end{multline*} 
If $\Ccal$ is an LTC, some hyperplane separates two vertices of the $3$-cube from the other vertices (corresponding to $x_{5}=1$ only for two points). These two vertices must be connected by an edge of the $3$-cube. 
Since $d_H(x',x'')=4$, four hyperplanes pass through this edge. 
There are only three different central hyperplanes through an edge of the $3$-cube, but four different central hyperplanes are required to produce $Z_{\pm,4}$. 
Hence $\Ccal$ is not a $(5,3)$-LTC. 
\end{example}

\subsection{Multi-covering numbers of hypercubes}\label{section6}

The previous section shows that the sets of strong modes realizable by $\RBM_{n,m}$ are related to the solution of the following problem: 
\begin{problem}
Let $m\leq n$. Consider an $m$-zonoset $\Zcal$ in $\R^n$ which does not intersect any two orthants separated by a single coordinate hyperplane.  How many orthants of $\R^n$ does $\Zcal$ intersect at most? 
\end{problem}

As we discuss in the following, this problem is related to the long standing problem of computing the {\em covering numbers} of hypercubes.  
\begin{definition}
The  {\em covering number}  of a hypercube is the smallest number of hyperplanes that slice each edge of the hypercube at least once. 
An edge is sliced by a hyperplane if the hyperplane intersects the relative interior of the edge and does not contain any vertices of the hypercube. 
A {\em cut} is the collection of all edges sliced by a hyperplane and corresponds to a linear threshold function. 
\end{definition}
The $m$ hyperplanes with normal vectors equal to the standard basis of $\R^m$ passing through the center of the $m$-dimensional hypercube slice all its edges. This arrangement is not always optimal. Paterson found $5$ hyperplanes slicing all edges of the $6$-cube, see~\cite{saks}. This shows that covering numbers do not behave trivially. The covering numbers are known only for hypercubes of dimension $\leq 6$. Computing them in higher dimensions is challenging, even in the cases where all cuts are known. Now: 
\begin{proposition}\label{coverpro}
If $Z_{+,n}$ is an $(n,n - 1)$-LTC, then 
there exists an arrangement of $n$ hyperplanes through the center of the $(n - 1)$-cube slicing each edge an even non-zero number of times. 
\end{proposition}
\begin{proof}
Given the assumption, there exists an arrangement of $n$ hyperplanes in $\R^{n-1}$ such that each vertex of the $(n-1)$-cube is in one of the $Z_{+,n}$-cells of the arrangement. 
Each vertex is separated by an even, positive number of hyperplanes from any other vertex, since any two elements of $Z_{+,n}$ differ in an even number of entries. Two vertices cannot be contained in the same cell of the arrangement, since $|Z_{+,n}|=2^{n-1}$ equals the total number of vertices of the $(n-1)$-cube. 
The code $Z_{+,n}$ is homogeneous, as  each coordinate $i\in[n]$ has the same number of zeros and ones, and hence each hyperplane in the arrangement can be chosen through the center of the cube. 
\end{proof}

Proposition~\ref{coverpro} motivates the following problem: 
\begin{problem}[Multi-covering number]
What is the smallest arrangement of hyperplanes, if one exists, that slices each edge of a hypercube a given number of times? 
\end{problem}

Of particular interest is the number of hyperplanes needed to slice each edge of the $m$-cube an even non-zero number of times. The edges of the $m$-cube can be sliced exactly twice by $2m$ hyperplanes  with normal vectors equal to the standard basis vectors of $\R^m$, counted with multiplicity two. 
Proposition~\ref{oddnot} shows that if $m$ is even and larger than zero, there is no arrangement of $(m+1)$ hyperplanes for which each vertex of the $m$-cube lies in a different cell and any two vertices are separated by an even number of hyperplanes. 
This suggests that when $m$ is even, there is no arrangement of $(m+1)$ hyperplanes slicing all edges of the $m$-cube exactly twice; at least not one for which each vertex lies in a different cell. 

There is exactly one way to slice all edges of the $3$-cube an even non-zero number of times by four hyperplanes, namely the way illustrated in Figure~\ref{slicescube}. To see that this is the only way, note that the $3$-cube has $12$ edges and that there are only $4$ different cuts that slice $6$ edges. 

The $4$-cube has $16$ vertices, $32$ edges, a total of $940$ different cuts, $3$ symmetry classes of central cuts, and $52$ different central cuts.  The maximal number of edges sliced by a cut is~$12$. Hence: 

\begin{proposition}
There is no arrangement of five hyperplanes, or less, slicing each edge of the four-dimensional cube at least twice.
\end{proposition}

The complexity of the next easiest example is considerable. 
We tested all combinations of six cuts of the $5$-cube and found: 
\begin{computation}\label{compu5cu}
There is no arrangement of six, or less, central hyperplanes slicing each edge of the five-dimensional cube an even  non-zero number of times. 
\end{computation}

In the following we explain some details of the computation. 
The $5$-cube has $80$ edges. 
There are $47\,285$ different ways of slicing them with affine hyperplanes, see~\cite{EmamyK20083156}. 
A cut is given by the indicator function on the set of edges sliced. A  list of the cuts can be found in~\cite{cubecuts}. 
An edge of the $m$-cube corresponds to a pair of binary vectors of length $m$ which differ in exactly one entry. Each edge is parallel to one coordinate vector of $\R^m$. The edges  can be organized in $m$ groups, corresponding to their directions. 
Within each group, the edges are naturally enumerated by the binary vectors of length $(m-1)$ containing the coordinate values that are equal for the two vertices of each edge. 
The central cuts can be characterized as the cuts which involve only pairs of opposite edges. 
The $5$-cube allows $7$ symmetry classes of central cuts and $941$ different central cuts. 
For each choice of $6$, or less, central cuts we computed the entry-wise addition of the indicator functions and found that this never produced an even non-zero value in each entry. 
On the other hand, $5$ is the covering number of the $5$-cube, see~\cite{EmamyK20083156}, and hence at least $6$ hyperplanes are needed to slice each edge twice.  
As a consequence of Computation~\ref{compu5cu} we have: 
\begin{corollary}\label{cor65}
The model $\RBM_{6,5}$ cannot represent any probability distribution with $32$ strong modes.
\end{corollary} 
Indeed, we trained $\RBM_{6,5}$ to approximate the uniform probability distribution on $Z_{+,6}$ and found a Kullback-Leibler divergence minimum of $0.6309$ (with base-two logarithm), which is a relatively large value. 
For this computation we used contrastive divergence~\cite{Hinton2002} and likelihood gradient with numerous parameter initializations.

\subsection[Proof of Theorem~1.6]{Proof of Theorem~\ref{thm:equivalences}}\label{section:summary}

The equivalence 
theorem from the introduction (illustrated in Figure \ref{fig:equivalences}) is a summary of observations from the previous subsections. 
For completeness we provide a proof. 
Recall the definition of {\bf LTC}, {\bf PR}, {\bf HP}, {\bf ZP}, {\bf SM}, and {\bf SP} given in Definition~\ref{def:summaryofproperties}. 
Let $n$ and $m$ be two integers and $\Ccal\subseteq\{0,1\}^n$. 

\begin{enumerate} 
\item 
The properties {\bf LTC}, {\bf HP,} and {\bf ZP} are equivalent. 

\smallskip
Let $W$ and $B$ be parameters making $\Ccal$ a linear threshold code, so that $\Ccal = \{\sgn ( h^\top W  +B)\colon h \in \{0,1\}^m \}$. 
Let $W_i$, $i=1,\ldots, n$ be the columns of $W$. 
The sign of $h^\top W_{i} +B_i$, $h \in \{0,1\}^m$ indicates which side of hyperplane $H_i$ in the arrangement $\Acal_{W,B}$ this $h$ lies on. 
Dually, the sign of $h^\top W_{i} +B_i$ indicates which side of the $i$-th coordinate hyperplane in $\R^n$ the point $h^\top W +B$ of the zonoset lies on. 

\item 
If $\Ccal$ satisfies {\bf PR} or {\bf SM}, then it is contained in an {\bf LTC} set. 

\smallskip
For $\Ccal$ to be the perfectly reconstructible, in particular it must be a subset of the image of a down inference function.   

If $\Ccal$ has the {\bf SM} property, its vectors are at least Hamming distance~$2$ apart.  By Theorem~\ref{sobrecomplmix}, each point in $\Ccal$ is the unique maximizer of a conditional distribution $p(\cdot|h)$ and an image point of the down inference function. 

\item 
If the vectors in $\Ccal$ are at least Hamming distance~$2$ apart, then {\bf SP} implies both {\bf SM} and {\bf PR}. 

\smallskip
{\bf SP} with Hamming distance two implies that there is a distribution $p\in\RBM_{n,m}$ with $p(v)>0$ for each $v\in\Ccal$, and $p(v')=0$ for each neighbor $v'$ of each $v\in\Ccal$. 
Therefore, each element of $\Ccal$ is a strong mode of $p$, and {\bf SM}. 
Writing $p_\theta(v,h)$ as a matrix with rows labeled by $h$, the Hamming distance two condition implies by Theorem~\ref{sobrecomplmix} that each row has a single non-zero entry, so $\down_\theta \circ \up_\theta$ is the identity on $\Ccal$, so {\bf PR} holds. 

\item 
If the vectors in $\Ccal$ are at least Hamming distance~$2$ apart and $\Ccal$ satisfies an $\ell_1$ property, then {\bf LTC} implies {\bf SP}.  

\smallskip
This is by Theorem~\ref{zonotopecondpropo}. 
\end{enumerate}


\section{Relative representational power}\label{section:when}
\subsection{When does a mixture of products contain an RBM?} \label{section3}
Using the characterizations obtained in Sections \ref{section1} and \ref{section2}, we now prove the result on relative representational power discussed in the introduction and depicted in Figure~\ref{solnm}. 
To do this, we derive upper bounds for the smallest $m$ such that $\RBM_{n,m}$ contains probability distributions with $l$ strong modes, and show thereby that RBMs can represent many more modes than mixtures of products with the same number of parameters.  

Any representability result, as Example~\ref{exrbm43}, combined with the following observation, yields lower bounds on the smallest mixture of products which contains the RBM model. 
\begin{observation}\label{composingzonotopes}
Let $k\in\mathbb{N}$. Assume that for each $i\in[k]$ there is a matrix $W^{(i)}\in\R^{m_i\times n_i}$ and a vector $B^{(i)}\in\R^{n_i}$ which generate a zonoset $\{h^\top W^{(i)}+B^{(i)}\colon h\in\{0,1\}^{m_i}\}$ intersecting $K_i$ even orthants of $\R^{n_i}$. 
Then 
\begin{equation*} 
W=\begin{pmatrix}W^{(1)} &&\\ &\ddots& \\ &&W^{(k)}\end{pmatrix}\quad\text{and}\quad B=(B^{(1)},\ldots,B^{(k)})
\end{equation*} 
generate a zonoset $\{h^\top W +B\colon h\in\{0,1\}^{m_1+\cdots+m_k}\}$ intersecting $\prod_{i\in [k]} K_i$ even orthants of $\R^{n_1+\cdots+n_k}$. 
\end{observation}

The following theorem provides the justification for the statement in the introduction that ``the number of parameters of the smallest mixture of products model containing an RBM model grows exponentially in the number of parameters of the RBM for any fixed ratio $0\!<\!m/n\!<\!\infty$,'' and for Figure~\ref{solnm}. 

\begin{theorem}\label{notinclpropodos}\mbox{}
Let $n,m\in\mathbb{N}$. 
\begin{itemize}
\item 
If $4\lceil m/3 \rceil \leq n$, then $\RBM_{n,m}\cap\,\mathcal{H}_{n,2^m}\neq\emptyset$ and 
$$\Mcal_{n,k}\supseteq\RBM_{n,m}  \text{ iff }  k\geq2^m.$$  
\item 
If $4\lceil m/3 \rceil  >  n$, then  $\RBM_{n,m}\cap\,\mathcal{H}_{n,L}\neq\emptyset$  and 
$$\Mcal_{n,k}\supseteq\RBM_{n,m}  \text{ only if } k\geq L,$$
where $L:=\min\{2^{l} + m-l , 2^{n-1}\}$, $l:=\max\{l\in\mathbb{N} \colon 4\lceil l/3 \rceil \leq n\}$. 
\end{itemize}
\end{theorem}

\begin{proof}
 Let $4\lceil m/3\rceil\leq n$. 
The {\em if} direction follows from  $\RBM_{n,m}\subseteq\Mcal_{n,2^m}$ for all $n$ and $m$. For the {\em only if} direction we show that $\RBM_{n,m}$ contains a probability distribution supported on a set of cardinality $2^m$ and minimum Hamming distance at least two (a distribution with $2^m$ strong modes). 
By Theorem~\ref{sobrecomplmix} such a distribution is in $\Mcal_{n,k}$ only if $k\geq 2^m$. 
Consider the following parameters: 
\begin{align*}
 W&=
\alpha {
\left(\begin{array}{c c c c c|}
     w &   &        &    &          \\
       & w &        &    &          \\
       &   & \ddots &    &          \\
       &   &        &  w &          \\
       &   &        &    & \tilde w \\
\end{array}
\begin{array}{c}
    \\
   \\
   0 \\
   \\
   \\ 
\end{array}
\right)
};
\quad
\begin{array}{c c l}
B&=& \alpha \left(b , b,  \ldots, b, -1, \ldots, -1 \right);\\ \\ 
b&=& \frac{1}{2} ( 3, 1, 1, 1) ;\\ \\
C&=&- W\frac{1}{2}(1,\ldots,1)^\top=
\alpha(1,\ldots,1)^\top, 
\end{array}
\end{align*}
where $\alpha\in\R$ is a constant, $w$ is the $3\times 4$-matrix defined in eq.~\eqref{weights43}, $\tilde w$ consists of the first or the first two rows of $w$, and $B$ is $\alpha$ times $\lceil m/3 \rceil$ copies of $b$ followed by $-1$s. 
Let $\lambda_i$ be the set of indices $\{1,2,3,4\}+4(i-1)\subset[n]$. 
For $\alpha\to\infty$ the visible distribution generated by $\RBM_{n,m}$ with parameters $W,B$, and $C$ is the uniform distribution on following subset of   
$Z_{+,n}$ of cardinality $2^m$: 
$\{ v\in \{0,1\}^n\colon \sum_{j\in\lambda_i}v_j \text{ is even for all }  i 
,\text{ and } v_j=0 \text{ for all }  j > 4\lceil m/3\rceil \}$.   

Now let $4\lceil m/3\rceil\geq n$. By the first part, $\RBM_{n,l}$ contains some $p$ with $2^l$ strong modes in $Z_{+,n}$. 
Moreover, $\RBM_{n,l+1}$ contains $\mu p + (1-\mu) \delta_x$ for any $p\in\RBM_{n,l}$, $x\in\{0,1\}^n$ and $\mu\in[0,1]$ (see~\cite{LeRoux2008}), such that each additional hidden unit can be used to increase the number of strong modes by one, until the set of strong modes is $Z_{+,n}$. 
\end{proof}

\begin{remark}
The statement of the first item of Theorem~\ref{notinclpropodos} remains true if $m=1\mod(3)$ and $4\lfloor m/3\rfloor +2 \leq n$. 
For $n<3$  we have $\Mcal_{n,k}=\RBM_{n,k-1}$ for any $k\in\mathbb{N}$. 
For $n=3$ we believe that $\Mcal_{3,3}$ and $\RBM_{3,2}$ are very similar, if not equal.  
\end{remark}

\subsection{When does an RBM contain a mixture of products?}\label{section5}

Complementary to question of when a mixture of products contains a product of mixtures, in this section we ask what is the smallest $m$ for which $\RBM_{n,m}$ contains $\Mcal_{n,k}$.  
We focus on an instance which we find particularly interesting: 
\begin{problem}\label{complpro}
Does $\RBM_{n,m}$ contain the mixture of products $\Mcal_{n,m+1}$?
\end{problem}

Both $\RBM_{n,m}$ and $\Mcal_{n,m+1}$ have $nm+n+m$ parameters and expected dimension $\min\{nm + n+m,2^n -1\}$.  
The expected dimension is also the true dimension of both models for most choices of $n$ and $m$~\cite{Catalisano2011,Cueto2010}. 
In the following we give a negative answer to Problem~\ref{complpro}. 

In the previous section  we showed that the non-negative rank of probability distributions in the model $\RBM_{n,m}$ is as large as $2^m$; there are tables of probabilities (probability distributions) 
represented by the RBM model, which cannot be represented as non-negative sums of less than $2^m$ non-negative rank-one tables (product distributions). 
The rank of a table $p$ is the smallest number $k$ such that $p$ can be written as a sum of $k$ rank-one tables. 
Here, a multivariate probability distribution $p=p(x_1,\ldots,x_n)$ with $x_i\in\X_i$, $|\X_i|=r_i$ for $i=1,\ldots,n$ is expressed as an $n$-way $r_1\times\cdots\times r_n$ table with value $p(x_1,\ldots,x_n)$ at the entry $(x_1,\ldots,x_n)$. 
A rank-one table is an outer-product of $n$ vectors of lengths $r_1, \ldots,r_n$. 
A product distribution in $\Mcal_{n,1}$ is the outer-product of the marginal distributions on the variables $x_1$ through $x_n$ and is a non-negative rank-one table. 
By definition, the elements of $\Mcal_{n,k}$ have non-negative rank at most $k$, and therefore also rank at most $k$. 
Since $\RBM_{n,m}$ is contained in $\Mcal_{n,2^m}$, any $p\in\RBM_{n,m}$ has rank at most $2^m$.

Two models $A$ and $B$  are called {\em generically distinguishable} if $A\cap B$ has relative measure zero in $A$ and in $B$. The restriction ``generically'' is useful, because in most cases of interest the models do intersect (e.g., mixtures of products and RBMs contain the uniform distribution). 
A {\em flattening} of  a table of probabilities is a way of arranging its entries in a two-way table (i.e., a matrix) by grouping the variables in two groups and considering the joint states of the variables in each of the groups as the states of two variables. The following is an example of a  flattening of a table $p$ for four binary variables: 
\begin{equation*}
p=\begin{pmatrix}
p_{00,00} & p_{00,01} & p_{00,10} & p_{00,11}\\
p_{01,00} & p_{01,01} & p_{01,10} & p_{01,11}\\
p_{10,00} & p_{10,01} & p_{10,10} & p_{10,11}\\
p_{11,00} & p_{11,01} & p_{11,10} & p_{11,11}
\end{pmatrix}.
\end{equation*}
The matrix rank of any flattening of a table $p$ is upper bounded by the outer-product rank of $p$. 
In particular, the vanishing of the $(k+1)\times(k+1)$-minors of flattenings are {\em algebraic invariants} of the model $\Mcal_{n,k}$. 

\begin{theorem}\label{mixtnotinrbm}
If $m\leq n/2$, then the model $\RBM_{n,m}$ contains points of rank $2^m$. 
If, furthermore, $m+1\neq 3$ or $n\neq 4$, then the models 
$\RBM_{n,m}$ and $\Mcal_{n,m+1}$ have dimension $nm +n +m$ and intersect at a set of dimension strictly less than $nm +n +m$.  \end{theorem} 
\begin{proof}
We show that if $m\leq n/2$, then  $\RBM_{n,m}$ contains a point $p$ with a flattening of rank $2^m$, which implies that $p$ has outer-product rank $2^m$. 
The flattenings of any $q\in \Mcal_{n,k}$ have rank at most $k$. 
This gives an algebraic invariant of the mixture of products model $\Mcal_{n,m+1}$ which is not satisfied by elements of $\RBM_{n,m}$. Hence, if both models have the same dimension $d$, then they intersect at a set of dimension strictly less than $d$. 

Consider the $m$-cube and the $2m$ hyperplanes through its center consisting of translates of the coordinate hyperplanes with multiplicity two. This hyperplane arrangement slices each edge of the $m$-cube exactly twice and generates a $(2m,m)$-LTC $\Ccal$ of minimum distance two.
The code $\Ccal$ consists of the $2^m$ binary vectors $x$ in $\{0,1\}^{2m}$ with $x_i=x_{i+1}$ for all odd $i$, and $\{(x_1,x_3,\ldots,x_{2m-1})\colon x\in\Ccal\}=\{0,1\}^m$. In the case $m=3$, for example, the code is 
\begin{equation*}
\Ccal= 
\begin{pmatrix} 
0&0&0&0&0&0 \\
1&1&0&0&0&0 \\
0&0&1&1&0&0 \\
0&0&0&0&1&1 \\
1&1&1&1&0&0 \\
1&1&0&0&1&1 \\
0&0&1&1&1&1 \\
1&1&1&1&1&1
\end{pmatrix}.
\end{equation*}
By Theorem~\ref{zonotopecondpropo}, $\RBM_{2m,m}$ contains the uniform distribution on $\Ccal$, $u_\Ccal$. View $u_\Ccal$ as a linear transformation from the $2^{m}$-dimensional space of real valued functions of $x_1,x_3,\ldots,x_{2m-1}$ to the space of functions of $x_2,x_4,\ldots,x_{2m}$. Then 
\begin{equation*}
u_\Ccal=\begin{pmatrix}
1/2^m &  &\\
 & 1/2^m &  & \\
 &  & \ddots & \\
 & &  & 1/2^m 
\end{pmatrix},
\end{equation*}
which has rank $2^m$. 
\end{proof}

\section{Discussion}
\label{sec:discussion}
RBMs create a multi-labeling of their input space by the most likely joint states of their hidden units given the inputs. 
The number of inference regions that can be generated in this way is of exponential order in the number of RBM parameters. 
The partitions of  $\R^n$ generated by an RBM with $n$ visible and $m$ hidden units can be identified with the intersections of affine spaces of dimension $d\leq \min\{n,m\}$ with the orthants of $\R^m$, whereby each affine space corresponds to a choice of the RBM parameters. 
We elaborated on the combinatorics of the resulting hyperplane arrangements, and on the combinatorics of point configurations in such hyperplane arrangements, in correspondence with the inference functions on the set of binary input vectors $\{0,1\}^n\subset\R^n$. 
Although the theory of hyperplane arrangements and linear separation of points is well studied in the literature, it still poses many questions (see examples below). 

We analyzed the sets of strong modes of probability distributions represented by RBMs and related them to the hyperplane arrangements and linear threshold codes (multi-labelings) mentioned above.  
The products of mixtures represented by RBMs are compact representations of probability distributions with many strong modes; of order $\min\{2^m,2^{n-1}\}$ for the RBM with $n$ visible and $m$ hidden units (exponential in the number of parameters). At the same time, Corollaries~\ref{corollaryoddnot} and~\ref{cor65} show that the hard bound $\min\{2^m,2^{n-1}\}$ is not always attained. 
Mixture models of product distributions (na\"ive Bayes models), on the other hand, generate less restricted input space partitions but into at most as many regions as mixture components, and can only represent probability distributions with a number of strong modes of linear order in the number of model parameters. 

These results imply that the smallest mixture model of product distributions that contains an RBM model is, in most cases, as large as one can possibly expect, having one mixture component per joint state of the RBM hidden units, and thus a number of parameters that is exponential in the number of RBM parameters. 
RBMs can represent distributions with many strong modes much more compactly than standard mixture models. 
This gives a concise combinatorial way of differentiating the two models. 
Fixing dimension, the RBMs which are hardest to represent as mixtures of product distributions are those with about the same number of visible and hidden units. 
At the same time, we note that there may exist small mixtures of product distributions which cannot be compactly represented by RBMs. 
For instance, Theorem~\ref{mixtnotinrbm} shows that $\Mcal_{n,m+1}\not\subseteq\RBM_{n,m}$ when $3\leq m\leq n/2$. 

These results aid our understanding of how models complement each other, and why distributed representations in  deep learning~\cite{Bengio-2009} can be expected to succeed, or when model selection can be based on theory  rather than trial-and-error. They confirm  the   intuition that distributed representations are exponentially more powerful than non-distributed ones,  in the case of binary RBMs and taking the number of strong modes, inference functions, and non-trivial perfectly reconstructible input sets as a measure of complexity. 
Other measures of complexity of probability distributions, such as  {\em multi-information}, which is defined as the Kullback-Leibler divergence to the set of  product distributions, are interesting but not necessarily best for differentiating between mixtures of products and RBMs. 
In terms of multi-information the most complex binary probability distributions have the form $p=\frac12 (\delta_x +\delta_y)$ with $x_i+y_i=1$ for all $i$, see~\cite{AyKnauf06:MaximizingMultiinformation}, and are contained in any (non-trivial) mixtures of products and  RBM models. 

Our approach has produced at least an order-of-magnitude or asymptotic understanding of the models we discuss.  We have also shown that to understand them fully would probably mean understanding as well some seemingly difficult equivalent combinatorial problems concerning linear threshold codes, hyperplane arrangements, multi-covering numbers, and so on.  On the other hand, we expect that our techniques can be used effectively to study more complex models in a similar fashion.

A number of problems is covered only partially by our analysis. 
Some interesting open cases include: 
\begin{itemize}
\item 
Computing multi-covering numbers for hypercubes of odd dimension larger than five. 
\item 
Characterizing the support sets of fully observable RBM models. This problem can be seen to be equivalent to characterizing the face lattices of polytopes defined as Kronecker products of hypercubes. 
\item 
Computing the maximal cardinality of linear threshold codes of minimum Hamming distance two. 
Are there cases where the first item of Theorem~\ref{notinclpropodos} holds for $4\lceil m/3 \rceil>n$, $m\leq n-1$, assuming $m\neq n-1$ when $n$ is odd? 
\item 
Can $\RBM_{8,7}$ represent probability distributions with $2^7$ strong modes? 
\item 
Verifying the conjecture that $\RBM_{3,2}\cap\Gcal_{3}=\emptyset$, and, in addition, proving or disproving $\Mcal_{3,3}=\RBM_{3,2}$. 
\item 
For $2< m < 2^{n-1}-1$, does $\RBM_{n,m}$ contain $\Mcal_{n,m}$? 
\end{itemize}

\subsubsection*{Acknowledgments}
The authors are grateful to Johannes Rauh for helpful discussions on algebraic invariants, to Nihat Ay for comments on hyperplane arrangements, to Yoshua Bengio for helpful discussions regarding distributed representations, and to Bernd Sturmfels for discussions motivating the analysis of mixtures of products in RBMs. 
The authors acknowledge use of the RCC-ITS computer cluster at the Pennsylvania State University. 
This work is supported in part by DARPA grant FA8650-11-1-7145.

\bibliographystyle{abbrv}
\bibliography{Literature}

\appendix

\bigskip

\begin{table}[h] 
	\begin{equation*}
	\underset{\tiny\begin{matrix}(000)&(011)&(101)&(110)&(001)&(010)&(100)&(111)\end{matrix}}{\begin{array}{cccccccc}
		1& 0& 0& 0& 0& 0& 0& 0\\
		0& 1& 0& 0& 0& 0& 0& 0\\
		0& 0& 1& 0& 0& 0& 0& 0\\
		0& 0& 0& 1& 0& 0& 0& 0\\
		\hline
		0& 1/4& 1/4& 1/4& 0& 0& 0& 1/4 \\
		1/4& 0& 1/4& 1/4& 0& 0& 1/4& 0\\
		1/4& 1/4& 0& 1/4& 0& 1/4& 0& 0\\
		1/4& 1/4& 1/4& 0& 1/4& 0& 0& 0\\
		\hline
		1/6& 1/6& 1/6& 1/6& 0& 1/6& 0& 1/6\\
		1/6& 1/6& 1/6& 1/6& 1/6& 0& 0& 1/6\\
		1/6& 1/6& 1/6& 1/6& 0& 0& 1/6& 1/6\\
		1/6& 1/6& 1/6& 1/6& 1/6& 1/6& 0& 0\\
		1/6& 1/6& 1/6& 1/6& 1/6& 0& 1/6& 0\\
		1/6& 1/6& 1/6& 1/6& 0& 1/6& 1/6& 0\\
		\hline
		1/7& 1/7& 1/7& 1/7& 0& 1/7& 1/7& 1/7\\
		1/7& 1/7& 1/7& 1/7& 1/7& 0& 1/7& 1/7\\
		1/7& 1/7& 1/7& 1/7& 1/7& 1/7& 0& 1/7\\
		1/7& 1/7& 1/7& 1/7& 1/7& 1/7& 1/7& 0\\
		\hline
		1/8& 1/8& 1/8& 1/8& 1/8& 1/8& 1/8& 1/8\\
		\end{array}} 
	\end{equation*}
	\caption{Vertex-presentation of the polytope $\overline{\Gcal_3^+}\subset\Pcal_3\subset\R^8$ (the set of probability distributions on $\{0,1\}^3$ with modes $Z_{+,3}=\{ (000) ,  (011) ,  (101) , (110)\}$). 
        Each row is a probability distribution that is a vertex of $\overline{\Gcal_3^+}$. 
		The vertices in the first group are the point measures on $Z_{+,3}$. 
		They have degree $18$ (i.e., they are incident to $18$ edges) and are connected by edges to all other vertices. 
		The vertices in the second and fourth groups have degree $11$. 
		There are no edges between pairs of vertices in the second group. 
		The vertices in the third group have degree $8$. 
		The uniform distribution has degree $12$. 
		The $f$-vector of the polytope, indicating the number of faces in each dimension, is $f(\overline{\Gcal}_3^+) = (19, 110, 290, 387, 270, 96, 16)$. 
		The volume is $\vol(\Gcal_3^+)/\vol(\Pcal_3)=1/56$. 
	}\label{table:vertices}
\end{table}

\begin{table}
\begin{equation*}
	\begin{array}{c c c c | c c c c | c c c c c c | c c c c | c}
		{}& 2& 3& 4& 5& {}& {}& 8& {}& 10& {}& 12& 13&  {}&  {}& 16& 17& 18& 19\\
		{}& 2& 3& 4& 5& {}& 7&  {}& 9&  {}& {}& 12&  {}& 14& 15&  {}& 17& 18& 19\\
		{}& 2& 3& 4& 5& 6& {}& {}& {}&  {}& 11& {}& 13& 14& 15& 16&  {}& 18& 19\\
		1& {}& 3& 4& {}& 6& {}& 8& {}& 10&  {}& 12& 13&  {}&  {}& 16& 17& 18& 19\\
		1& {}& 3& 4& {}& 6& 7& {}& 9&  {}&  {}& 12&  {}& 14& 15&  {}& 17& 18& 19\\
		1& {}& 3& 4& 5& 6& {}& {}& 9& 10& 11&  {}&  {}&  {}& 15& 16& 17&  {}& 19\\
		1& 2& {}& 4& {}& {}& 7& 8& {}& 10&  {}& 12& 13&  {}&  {}& 16& 17& 18& 19\\
		1& 2& {}& 4& {}& 6& 7&  {}& {}&  {}& 11&  {}& 13& 14& 15& 16&  {}& 18& 19\\
		1& 2& {}& 4& 5& {}& 7&  {}& 9& 10& 11&  {}&  {}&  {}& 15& 16& 17&  {}& 19\\
		1& 2& 3& {}& {}& {}& 7& 8& 9&  {}&  {}& 12&  {}& 14& 15&  {}& 17& 18& 19\\
		1& 2& 3& {}& {}& 6& {}& 8& {}&  {}& 11&  {}& 13& 14& 15& 16&  {}& 18& 19\\
		1& 2& 3& {}& 5& {}& {}& 8& 9& 10& 11&  {}&  {}&  {}& 15& 16& 17&  {}& 19\\
		1& 2& 3& 4& {}& 6& 7& 8& {}&  {}&  {}& 12& 13& 14&  {}&  {}&  {}&18&  {}\\
		1& 2& 3& 4& 5& 6& 7& {}& 9&  {}& 11&  {}&  {}& 14& 15&  {}&  {}&  {}& {}\\
		1& 2& 3& 4& 5& 6& {}& 8& {}& 10& 11&  {}& 13&  {}&  {}& 16& {}& {}&  {}\\
		1& 2& 3& 4& 5& {}& 7& 8& 9& 10&  {}& 12&  {}&  {}&  {}&  {}& 17& {}&  {}\\ 
	\end{array} 
\end{equation*}
	\caption{
		Vertex-facet incidence table for the polytope $\overline{\Gcal_3^+}$. 
		Each row gives the list of vertices incident to one facet of $\overline{\Gcal_3^+}$ 
		(the index of each vertex corresponds to the row in which it appears in Table~\ref{table:vertices}). 
		The set of vertices in any face of the polytope is an intersection of some of the $16$ sets listed above. 
	}\label{table:facets}
\end{table}

\end{document}